\documentclass{article}
\usepackage{bm, prettyref, subcaption, algorithm, algorithmic}
\usepackage{float}
\usepackage{wrapfig}
\usepackage{amssymb, amsthm, amsmath}
\usepackage{ifthen}

\usepackage[pdftex]{graphicx}
\usepackage[utf8]{inputenc} 
\usepackage[T1]{fontenc}    
\usepackage{hyperref}       
\usepackage{url}            
\usepackage{booktabs}       
\usepackage{amsfonts}       
\usepackage{nicefrac}       
\usepackage{microtype}      
\usepackage{xcolor}

\usepackage{fullpage} 

\newrefformat{fig}{Fig.~\ref{#1}}
\newrefformat{cor}{Corollary ~\ref{#1}}

\theoremstyle{plain}
\newtheorem{theorem}{Theorem}[section]

\newtheorem{lemma}[theorem]{Lemma}
\newtheorem{corollary}[theorem]{Corollary}
\theoremstyle{definition}
\newtheorem{definition}[theorem]{Definition}

\theoremstyle{remark}
\newtheorem{remark}[theorem]{Remark}

\usepackage{xspace,exscale,relsize}
\usepackage{fancybox,shadow}

\usepackage{color}
\usepackage{amsfonts}

\usepackage[title]{appendix}

\usepackage{mathtools}

\makeatletter
\let\over=\@@over \let\overwithdelims=\@@overwithdelims
\let\atop=\@@atop \let\atopwithdelims=\@@atopwithdelims
\let\above=\@@above \let\abovewithdelims=\@@abovewithdelims
\makeatother
\usepackage{rotating}

\usepackage{float}

\usepackage{ifpdf}

\usepackage{psfrag}

\usepackage{prettyref,enumerate}

\usepackage{tikz}
\usetikzlibrary{arrows}
\tikzstyle{int}=[draw, fill=blue!20, minimum size=2em]
\tikzstyle{dot}=[circle, draw, fill=blue!20, minimum size=2em]
\tikzstyle{init} = [pin edge={to-,thin,black}]

\usepackage[all]{xy}

\newcommand{\Regret}{\mathsf{Regret}}

\ifx\eqref\undefined
\newcommand{\eqref}[1]{~(\ref{#1})}
\fi
\ifx\mod\undefined
\def\mod{\mathop{\rm mod}}
\fi

\usepackage{bm}

\newcommand{\argmin}{\mathop{\rm argmin}}
\newcommand{\argmax}{\mathop{\rm argmax}}

\def\EE{\Expect}

\makeatletter
\def\bbordermatrix#1{\begingroup \m@th
	\@tempdima 4.75\p@
	\setbox\z@\vbox{%
		\def\cr{\crcr\noalign{\kern2\p@\global\let\cr\endline}}%
		\ialign{$##$\hfil\kern2\p@\kern\@tempdima&\thinspace\hfil$##$\hfil
			&&\quad\hfil$##$\hfil\crcr
			\omit\strut\hfil\crcr\noalign{\kern-\baselineskip}%
			#1\crcr\omit\strut\cr}}%
	\setbox\tw@\vbox{\unvcopy\z@\global\setbox\@ne\lastbox}%
	\setbox\tw@\hbox{\unhbox\@ne\unskip\global\setbox\@ne\lastbox}%
	\setbox\tw@\hbox{$\kern\wd\@ne\kern-\@tempdima\left[\kern-\wd\@ne
		\global\setbox\@ne\vbox{\box\@ne\kern2\p@}%
		\vcenter{\kern-\ht\@ne\unvbox\z@\kern-\baselineskip}\,\right]$}%
	\null\;\vbox{\kern\ht\@ne\box\tw@}\endgroup}
\makeatother

\newcommand{\Expect}{\mathbb{E}}

\newcommand{\indc}[1]{{\mathbf{1}_{\left\{{#1}\right\}}}}

\definecolor{myblue}{rgb}{.8, .8, 1}
\definecolor{mathblue}{rgb}{0.2472, 0.24, 0.6} 
\definecolor{mathred}{rgb}{0.6, 0.24, 0.442893}
\definecolor{mathyellow}{rgb}{0.6, 0.547014, 0.24}

\newcommand{\mmse}{\mathsf{mmse}}

\def\E{\mathbb{E}}

\newrefformat{eq}{(\ref{#1})}
\newrefformat{thm}{Theorem~\ref{#1}}
\newrefformat{th}{Theorem~\ref{#1}}
\newrefformat{chap}{Chapter~\ref{#1}}
\newrefformat{defn}{Definition~\ref{#1}}
\newrefformat{sec}{Section~\ref{#1}}
\newrefformat{seca}{Section~\ref{#1}}
\newrefformat{algo}{Algorithm~\ref{#1}}
\newrefformat{fig}{Fig.~\ref{#1}}
\newrefformat{tab}{Table~\ref{#1}}
\newrefformat{rmk}{Remark~\ref{#1}}
\newrefformat{clm}{Claim~\ref{#1}}
\newrefformat{def}{Definition~\ref{#1}}
\newrefformat{cor}{Corollary~\ref{#1}}
\newrefformat{lmm}{Lemma~\ref{#1}}
\newrefformat{prop}{Proposition~\ref{#1}}
\newrefformat{pr}{Proposition~\ref{#1}}
\newrefformat{property}{Property~\ref{#1}}
\newrefformat{app}{Appendix~\ref{#1}}
\newrefformat{apx}{Appendix~\ref{#1}}
\newrefformat{ex}{Example~\ref{#1}}
\newrefformat{exer}{Exercise~\ref{#1}}
\newrefformat{soln}{Solution~\ref{#1}}
\newrefformat{alg}{Algorithm~\ref{#1}}

\def\unifto{\mathop{{\mskip 3mu plus 2mu minus 1mu%
			\setbox0=\hbox{$\mathchar"3221$}%
			\raise.6ex\copy0\kern-\wd0%
			\lower0.5ex\hbox{$\mathchar"3221$}}\mskip 3mu plus 2mu minus 1mu}}

\ifx\lesssim\undefined
\def\simleq{{{\mskip 3mu plus 2mu minus 1mu%
			\setbox0=\hbox{$\mathchar"013C$}%
			\raise.2ex\copy0\kern-\wd0%
			\lower0.9ex\hbox{$\mathchar"0218$}}\mskip 3mu plus 2mu minus 1mu}}
\else
\def\simleq{\lesssim}
\fi

\ifx\gtrsim\undefined
\def\simgeq{{{\mskip 3mu plus 2mu minus 1mu%
			\setbox0=\hbox{$\mathchar"013E$}%
			\raise.2ex\copy0\kern-\wd0%
			\lower0.9ex\hbox{$\mathchar"0218$}}\mskip 3mu plus 2mu minus 1mu}}
\else
\def\simgeq{\gtrsim}
\fi

\newif\ifmapx
{\catcode`/=0 \catcode`\\=12/gdef/mkillslash\#1{#1}}
\edef\jobnametmp{\expandafter\string\csname ic_apx\endcsname}
\edef\jobnameapx{\expandafter\mkillslash\jobnametmp}
\edef\jobnameexpand{\jobname}
\ifx\jobnameexpand\jobnameapx
\mapxtrue
\else
\mapxfalse
\fi

\title{Solving Empirical Bayes via Transformers}

    \author{Anzo Teh, Mark Jabbour, Yury Polyanskiy\thanks{
    M.J. was with the Department of EECS, MIT, Cambridge,
		MA, email: \url{mjabbour@mit.edu}. 
    A.T. and Y.P. are with the Department of EECS, MIT, Cambridge,
		MA, email: \url{anzoteh@mit.edu} and \url{yp@mit.edu}. Code available at \url{https://github.com/Anzoteh96/eb-transformers}}}

\begin{document}

\maketitle

\begin{abstract}
This work applies modern AI tools (transformers) to solving one of the oldest statistical problems: Poisson means under empirical Bayes (Poisson-EB) setting. In Poisson-EB, a high-dimensional mean vector $\theta$ (with i.i.d. coordinates sampled from an unknown prior $\pi$) is estimated on the basis of $X=\mathrm{Poisson}(\theta)$. A transformer model is pre-trained on a set of synthetically generated pairs $(X,\theta)$ and learns to do in-context learning (ICL) by adapting to unknown $\pi$. Theoretically, we show that a sufficiently wide transformer can achieve vanishing regret with respect to an oracle estimator that knows $\pi$ as the dimension grows to infinity. Practically, we discover that already very small models (100k parameters) are able to outperform the best classical algorithm (non-parametric maximum likelihood, or NPMLE) both in runtime and validation loss, which we compute on out-of-distribution synthetic data as well as real-world datasets (NHL hockey, MLB baseball, BookCorpusOpen). Finally, by using linear probes, we confirm that the transformer's EB estimator appears to internally work differently from either NPMLE or Robbins' estimators.
\end{abstract}

\tableofcontents

\section{Introduction}\label{sec:intro}
Transformers have received a lot of attention due to the prevalence of large language models (LLM). More generally, we think of (encoder-only) transformers as generic engines for learning from exchangeable data. Since most classical statistical tasks are formulated under the i.i.d. sampling assumption, it is very natural to try to apply transformers to them~\cite{garg2022can}. 

Training transformers for classical statistical problems serves two purposes. One is obviously to get better estimators. Another, equally important, goal of such exercises is to elucidate the internal workings of transformers in a domain with a much easier and much better understood statistical structure than NLP. 
In this work, we believe, we found the simplest possible such statistical task: \textit{empirical Bayes (EB) mean estimation}. 
We believe transformers are suitable for EB because EB estimators naturally exhibit a shrinkage effect (i.e. biasing mean estimates towards the nearest mode of the prior), and so do transformers, as shown in \cite{geshkovski2024emergence} that the attention mechanisms tend to cluster tokens. 
Additionally, the EB mean estimation problem is permutation equivariant, removing the need for positional encoding. 
In turn, estimators for this problem are in high demand 
\cite{koenker2024empirical, gu2023invidious, gu2022nonparametric} and unfortunately, the best classical estimator (so-called non-parametric maximum likelihood, or NPMLE) suffers from slow convergence. In this work, we demonstrate that transformers outperform NPMLE while also running almost 100x faster. 
We now proceed to defining the EB task.

\textit{Poisson-EB task:} One observes $n$ samples $X_1,\ldots,X_n$ which are generated iid via a two-step process. First, $\theta_1,\ldots,\theta_n$ are sampled from some unknown prior $\pi$ on $\mathbb{R}$. The $\pi$ serves as an unseen (non-parametric) latent variable, and we assume nothing about it (not even continuity or smoothness). Second, given $\theta_i$'s, we sample $X_i$'s conditionally iid via $X_i \sim \text{Poi}(\theta_i)$. The goal is to estimate $\theta_1, \cdots, \theta_n$ via $\hat{\theta}_1, \cdots, \hat{\theta}_n$ upon seeing $X_1, \cdots, X_n$ that minimizes the expected mean-squared error (MSE), $\mathbb{E}[(\hat{\theta}(X) - \theta)^2]$. 
If $\pi$ were known, the Bayes estimator that minimizes the MSE is the posterior mean of $\theta$, which also has the following form. 
\begin{equation}\label{eq:poisson-bayes}
\hat{\theta}_{\pi}(x) = \mathbb{E}[\theta | X = x] = (x + 1)\frac{f_{\pi}(x + 1)}{f_{\pi}(x)}\,.
\end{equation}
where $f_{\pi}(x)\triangleq \mathbb{E}_{\pi}[e^{-\theta}\frac{\theta^x}{x!}]$ is the posterior density of $x$. 
Given that $\pi$ is unknown, an estimator $\pi$ can only instead approximate $\hat{\theta}_{\pi}$. 
We quantify the quality of the estimation as the \emph{regret}, defined as the excess MSE of $\hat{\theta}$, over $\hat{\theta}_{\pi}$. 
\[
        \Regret(\hat{\theta}) = \E\left[\left(\hat{\theta}(X)-\theta\right)^2\right] - \E\left[\left(\theta_{\pi}(X)-\theta\right)^2\right] 
        = \E\left[\left(\hat{\theta}(X)-\theta\right)^2\right] - \mmse(\pi)
\]

In this Poisson-EB setting, multiple lines of work have produced estimators that resulted in regret that vanishes as sample size increases \cite{brown2013poisson, polyanskiy2021sharp, jana2022optimal, jana2023empirical}. 
Robbins estimator \cite{Rob51, Rob56} replaces the unknown posterior density $f_{\pi}$ in \prettyref{eq:poisson-bayes} with $N_n(\cdot)$, the empirical count among the samples $X_1, \cdots, X_n$. 
Minimum distance estimators first estimate a prior (e.g. the NPMLE estimator $\hat{\pi}_{\mathsf{NPMLE}} = \argmax_{Q} \prod_{i=1}^n f_Q(X_i)$), 
and then produces the plugged-in Bayes estimator $\hat{\theta}_{\hat{\pi}}$.  
Notice that Robbins estimator suffers from multiple shortcomings like numerical instability (c.f. \cite{efron2021computer}) and the lack of monotonicity property of the Bayes estimator $\hat{\theta}_{\pi}$ 
(c.f. \cite{houwelingen1983monotone}), 
while minimum-distance estimators are too computationally expensive and do not scale to higher dimensions. 
\cite{jana2023empirical} attempts to remedy the `regularity vs efficiency' tradeoffs in these estimators with an estimator based on score-estimation equivalent in the Poisson model. 
However, despite the monotone regularity added, this estimator still does not have a Bayesian form: a cost one pays to achieve an efficient computational time.

\textit{Solving Poisson-EB via transformers.} 
We formulate our procedure for solving Poisson-EB as follows: we generate synthetic data and train our transformers on them. Then, we freeze their weights and present new data to be estimated. 
Concretely, our contributions are as follows: 
\begin{enumerate}
    \item In \prettyref{sec:theory}, we show that transformers can approximate Robbins and the NPMLE via the universal approximation theorem. We also use linear probes to show that our pre-trained transformers appear to work differently from the two aforementioned classical EB estimators. 

    \item In \prettyref{sec:synthetic}, we set up synthetic experiments to demonstrate that  pre-trained transformers can generalize to unseen sequence lengths and evaluation priors. 
    This is akin to \cite{xie2021explanation} where ICL occurs at test time despite distribution mismatch. 
    In addition, we will show that a transformer of reduced complexity can perform at almost the same level. 

    \item In \prettyref{sec:real}, we evaluate these transformers on real datasets for a similar prediction task to demonstrate that they often outperform the classical baselines and crush them in terms of speed.
\end{enumerate}

One interesting finding is that our transformers demonstrate \textit{length-generalization}, whereas previously transformers' track record in this regard was rather mixed \cite{zhou2024transformers, wang2024length, kazemnejad2024impact, anil2022exploring}. Specifically,  our transformers achieve length-generalization by getting lower regret upon being tested on sequence lengths up to 4x the length they were trained on, even in zero-shot setting where the data comes from unseen type of prior. 

We mention that there is a long literature studying transformers for many statistical problems~\cite{bai2024transformers}, further reviewed below. What distinguishes our work is a) our estimator does not just match but improve upon existing (classical) estimators, thus advancing the statistical frontier; b) our setting is unsupervised and  closer to NLP compared to most previous work considering supervised learning (classification and regression), in which data comes in \textit{pairs}, thus requiring unnatural tricks to pair tokens; c) our problem is non-parametric.

In summary, we demonstrate that even for classical statistical problems, transformers offer an excellent alternative (in runtime and performance). 
For the simple 1D Poisson-EB task, we also found that already very economically sized transformers ($< 100$ k parameters) can have excellent performance. We hope that disecting the inner workings of these transformers will suggest novel types of statistical procedures for the large-scale statistical inference~\cite{efron2012large}.

\section{Related work}
\textbf{Transformers and in-context learning (ICL). }
Transformers have shown the ability to do ICL, as per the thread of work summarized in 
\cite{dong2022survey}. 
ICL is primarily manifested in natural language processing \cite{brown2020language, dai2022can}; 
several works also study the ability and limitations of transformers in using ICL in doing various statistical tasks 
\cite{akyurek2022learning, zhang2023trained, garg2022can, tarzanagh2023transformers, bhattamishra2023understanding, akyurek2022learning, bai2024transformers}. 
Recent works have also explained ICL from a Bayesian point of view 
\cite{muller2024bayes, panwar2023context}, including upon train-test distribution mismatch \cite{xie2021explanation}.

\textbf{How do transformers work?} 
\cite{yun2019transformers} has established the universal approximation theorem of transformers. 
This was later extended for sparse transformers \cite{yun2020n} and the ICL setting \cite{furuya2024transformers}. Its limitations are further discussed in \cite{nath2024transformers}. 
Transformers have also been shown to do other approximation tasks, like Turing machines 
\cite{wei2022statistically, perez2021attention}. 
From another perspective, 
\cite{alain2018understanding} introduces linear probes as a mechanism of understanding the internals of a neural network. 
Linear probe has also been applied in transformers to study its ability to perform NLP tasks \cite{tenney2019bert}, achieve second order convergence \cite{fu2024transformers}, learn various functions in-context \cite{guo2023transformers}, and improve in-context learning \cite{abbas2024enhancing}. 

\textbf{Empirical Bayes. }
Empirical Bayes was introduced in \cite{Rob56} based on the premise that estimation tasks on a sequence can yield lower risk by allowing estimation of each individual component to depend on the entire sequence, e.g. as demonstrated by the James-Stein estimator \cite{stein1956inadmissibility, james1961estimation}. 
In the Poisson model, known estimators for the mean estimation problem are based either on the Tweedie formula \cite{Rob56}, posterior density estimation \cite{kiefer1956consistency, lindsay1983geometry}, 
or ERM \cite{jana2023empirical}. These methods have been shown to yield minimax optimal regret for certain classes of priors 
\cite{brown2013poisson, polyanskiy2021sharp, jana2022optimal, jana2023empirical}. 
In practice, empirical Bayes is a powerful tool for large-scale inference \cite{efron2012large}, 
with applications including  downstream tasks like linear regression 
\cite{kim2024flexible, mukherjee2023mean}, estimating the number of missing species \cite{fisher1943relation}, large scale hypothesis testing \cite{efron2001empirical}, 
and the construction of biological sequencing frameworks 
\cite{hardcastle2010bayseq, leng2013ebseq}. 

\textbf{Comparison to Bayesian methods and SBI.} We compare our approach against simulation-based inference (SBI) \cite{cranmer2020frontier} in the fully Bayesian model. In both EB and Bayesian model, we desire to obtain posterior $p(\theta | X)$, which requires knowledge of a good prior $\pi$ on $\theta$. In SBI,  this is resolved via a method of amortized inference \cite{zammit2024neural}, where one trains a model on many $(\theta, x)$ pairs thus implicitly learning the prior (and often the forward model $p(X|\theta)$ as well). It is hoped that the learned prior matches the one during inference. 
We note that transformers have been applied in Bayesian settings \cite{hollmann2025accurate, gloeckler2024all, chang2024amortized}, but these applications require either labeled (supervised) data at inference time or simulation access to the prior $\pi$ itself. Unsurprisingly, prior misspecification causes serious problems~\cite[Appendix S2.8]{zammit2024neural} unless one uses some regularization
\cite{huang2023learning}. 

In contrast, in the EB setting, absolutely nothing about the prior is assumed at inference time, with the model instead trying to estimate $\mathbb{E}[\theta | X = x]$ directly from the data consisting of $X_i$'s only. We stress again, that in EB there are no labeled (no $\theta$'s!) data presented to the estimator at inference time. Thus, while our model is trained on synthetic examples sampled from a variety of priors, the prior at inference time is expected to deviate from those used at pre-training and the whole goal of our work is to improve the zero-shot generalization ability.
We will discuss more in \prettyref{app:bayesian}, using illustrations of real-life datasets where this comparison becomes more apparent.

\section{Preliminaries}\label{sec:task}
\subsection{Baselines description}\label{sec:baselines}
We outline some of the classical algorithms that we will be benchmarking against. 

\textbf{Non empirical Bayes baselines.} 
When nothing is known about the prior $\pi$, the minimax optimal estimator is the familiar maximum-likelihood (MLE) estimator
$\hat{\theta}_{\mathsf{MLE}}(x) = x$. However, when one restricts priors in some way, the minimax optimal estimator is not MLE, but rather a Bayes estimator for the \emph{worst-case} prior. In this work, we consider priors restricted to support $[0,50]$. The minimax optimal estimator for this case is referred to as the \emph{gold standard} (GS) estimator to signify its role as the ``best'' in the sense of classical (pre-EB) statistics.
\prettyref{app:worstprior} contains derivation of GS.

\textbf{Empirical Bayes baselines.}
We will use the following empirical Bayes estimators as introduced in \prettyref{sec:intro}: 
the Robbins estimator, NPMLE estimator, and the ERM-monotone estimator with the algorithm described in Lemma 1 of  \cite{jana2023empirical}.

    \subsection{Transformer architecture}\label{sec:architecture}
    Next, we describe our transformer architecture, 
    which closely mimics the standard transformer architecture in \cite{vaswani2017attention}. 
    Given the permutation invariance of the Bayes estimator, we do not use positional encoding or masking. 
    Thus effectively, it is a full-attention encoder-only transformer with one linear decoder on top. 
    In addition, we achieve parameter sharing by using \emph{two} different weights, split evenly across the $N$ layers. 
    The intuition behind it is that one learns the encoding part (input) and the other the decoding part (output). 

    \textbf{Linear attention}. 
    We considered the following linearized attention using some feature representation $\phi$: 
    \begin{equation}\label{eq:linear_attn_defn}
    \text{Attention}(Q, K, V) = \frac 1n \phi(Q)(\phi(K)^TV)
    \end{equation}
    We experimented with several choices of $\phi$ and found that the identity function $\phi(x)=x$ works well. 
    
    Linearized attention is formalized in \cite{katharopoulos2020transformers}, 
    and subsequently refined in frameworks like Mamba \cite{gu2023mamba, dao2024transformers}, 
    gated update rule \cite{yang2023gated}, and Delta rule 
    \cite{yang2024gated, yang2024parallelizing}.
    While we will show that the attention defined in \prettyref{eq:linear_attn_defn} performs almost as well in our EB setting,
    a rigorous comparison of different linear attention frameworks in solving EB will be subject to future work. 

    \subsection{Training protocol}\label{sec:training}
    \textbf{Data generation.}
    We emphasize that all our transformers are trained on synthetic data, using the Poisson-generated integers $X$ as inputs and the hidden parameters $\theta$ as labels.
    We train our transformers via the MSE loss $\sum_{i=1}^n (\hat{\theta}(X_i) - \theta_i)^2$ and the Adam optimizer \cite{kingma2014adam}. 

    There are two classes of priors from which we generate $\theta_{\text{base}}\in [0, 1]$: 
    the neural-generated prior-on-priors $\in\mathcal{P}([0, 1])$, and Dirichlet process with 
    base distribution $\mathsf{Unif}[0, 1]$ within each batch. These are further described in 
    \prettyref{app:train-priors}. 
    We then define $\theta = \theta_{\text{base}} \cdot \theta_{\max}$. 
    At evaluation stage, $\theta_{\max}$ is generally fixed (e.g. 50 in \prettyref{sec:synthetic}). 
    To ensure robustness of our transformers against inputs of different magnitude, however, at training stage we let $\theta_{\max}$ be random, sampled from 
    \[
    \frac 34\text{Unif}([0, 200]) + \frac 18\text{Exp}(50) + \frac 18\text{Cauchy}(50, 10)
    \]
    and capped at $\theta_{\max} \le 500$. 
    We defer the detailed discussion to \prettyref{app:train-priors}, including the motivation to train with a mixture of the two priors. 
    
    \textbf{Parameter selection.}
We consider models of 6, 12, 18, 24, and 48 layers, embedding dimension $\mathsf{dmodel}$ either 32 or 64, 
and number of heads in 4, 8, 16, 32. 
We fix the number of training epochs to 50k, the learning rate to 0.02, and the decay rate every 300 epochs to 0.9. We train our transformers on sequence length 512, and for each epoch the we use 192 batches. 
Among the trained models, we chose our models based on the mean-squared error evaluated on neural prior-on-prior and Dirichlet process during inference time. We eventually arrive at our model T24r, 
with 24 layers, embedding dimension 32, 8 heads. 
We will make comparison against L24r, which is its linear attention counterpart as defined in \prettyref{eq:linear_attn_defn}, and with the identity feature representation $\phi(x)=x$. 

\section{Understanding transformers}\label{sec:theory}
In this section, we try to gain an intuition on how transformers work in the empirical Bayes setting. 
We achieve this from two angles. First, we establish some theoretical results on the expressibility of transformers in solving empirical Bayes tasks. 
Second, we use linear probes to study the prediction mechanism of the transformers. 

\subsection{Expressibility of transformers}\label{sec:theorems}
We discuss the feasibility of using transformers to solve the empirical Bayes prediction task. Universal learnability of transformers is first established in \cite{yun2019transformers}, 
and further characterized in \cite{furuya2024transformers}. 

To start with, we consider the clipped Robbins estimator, defined as follows: 
\begin{equation}
    \hat{\theta}_{\mathsf{Rob}, d, M}(x) = 
    \begin{cases}
        \min \{(x + 1)\frac{N(x + 1)}{N(x)}, M\} & x < d\\
        M & x\ge d
    \end{cases}
\end{equation}
Here, we show that transformers can learn this clipped Robbins estimator up to an arbitrary precision. 

\begin{theorem}\label{thm:robbins-transformers}
    Set a positive integer $d$ and a positive real number $M$. 
    Then for any $\epsilon > 0$, there exists a transformer  
    that learns the clipped Robbins estimator 
    $\hat{\theta}_{\mathsf{Rob}, d, M}$ up to a precision $\epsilon$, 
    In addition, this transformer has embedding dimension $d+1$ at encoding stage, 
    and a two-layered feedforward network with $O((1+M)^2\epsilon^{-1})$ hidden neurons at the hidden layer at decoding stage. 
\end{theorem}

We also show that transformers can approximate NPMLE up to an arbitrary input value and precision, although we will omit quantitative bounds on this transformer due to challenges stated in \cite[Remark 3]{furuya2024transformers}. 

\begin{theorem}\label{thm:univ_npmle}
    Let $M > 0$, and denote the NPMLE estimator $\hat{\theta}_{\mathsf{NPMLE}, M}$, the NPMLE estimator chosen among 
    $\mathcal{P}([0, M])$. For each integer $d > 0$, consider the following modified NPMLE function: 
    \[
    \theta_{\mathsf{NPMLE}, d, M}(x)
    =
    \begin{cases}
        \hat{\theta}_{\mathsf{NPMLE}}(x) & x \le d\\
        M & x > d\\
    \end{cases}
    \]
    then for any $\epsilon > 0$ there exists a transformer that can approximate $\theta_{\mathsf{NPMLE}, d}$ uniformly up to $\epsilon$-precision. 
\end{theorem}

Full proofs are deferred to \prettyref{app:proofs}, and we only give a sketch for now. 
For Robbins approximation, we use softmax attention to create an encoding mechanism that encodes $\frac{N(X_i)}{N(X_i)+(X_i + 1)N(X_i + 1)}$ at position $i$ among $1, \cdots, n$ and use a decoder to approximate the function $x\to \min\{\frac{1}{x} - 1, M\}$. For NPMLE approximation, we pass in the Sigmoid of the integer inputs as embedding, 
and show that $\hat{\theta}_{\mathsf{NPMLE}}$ can be continuously extended, with sigmoid-transformed empirical distribution as arguments. Then by \cite[Theorem 1]{furuya2024transformers}, there exists a transformer that can learn NPMLE to an arbitrary precision. For the encoding part, we provide a pseudocode in \prettyref{app:transf-robbins} that closely follows PyTorch's implementation. 

To illustrate the significance of both of these theorems, we demonstrate that transformers can learn an empirical Bayes prediction task to an arbitrarily low regret, 
although the quantitative bounds for the transformer size only apply for Robbins estimation. 

\begin{corollary}\label{cor:vanish-regret}
    Let $\theta_{\max} > 0$ be fixed. 
    For any $\epsilon > 0$, there exists an integer $N = O_{\theta_{\max}}(\epsilon^{-1}\log^2(\epsilon^{-1})))$ and a transformer $\Gamma$ with $O_{\theta_{\max}}(\epsilon^{-1})$ parameters such that for all $n\ge N$, 
    the minimax regret of $\Gamma(X_1, \cdots, X_n)$ on prior $\pi\in\mathcal{P}([0, \theta_{\max}])$ satisfies 
    \[
    \sup_{\pi\in\mathcal{P}([0, \theta_{\max}])}\mathsf{Regret}(\Gamma(X_1, \cdots, X_n))\le \epsilon
    \]
\end{corollary}

\subsection{How do transformers learn?}
We study the mechanisms by which transformers learn via linear probe \cite{alain2018understanding}. 
To this end, we take the representation of each layer of our pretrained transformers, and train a decoder that comprises a layer normalization operation, a linear layer, and GeLU activation. 
This decoder is then trained with the following labels: frequency $N(x)$ within a sequence, 
and posterior density $f_{\hat{\pi}}(x)$ estimated by the NPMLE. 
The aim is to study whether our transformers function like the Robbins or NPMLE. 
In the plot in \prettyref{fig:linear_probe}, we answer this as negative, showing that our transformers are not merely learning about these features, but instead learning what the Bayes estimator $\hat{\theta}_{\pi}$ is.

\begin{figure}[ht]
\vskip 0.2in
\begin{center}
\begin{subfigure}[b]{0.3\linewidth}
    \includegraphics[width=\linewidth]{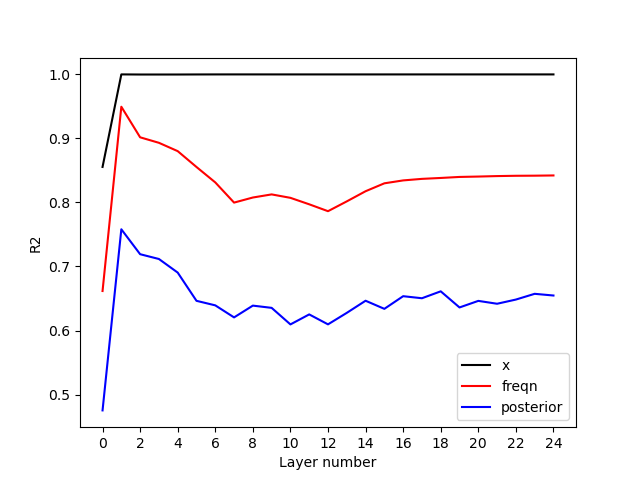}
    \caption{Neural prior}
\end{subfigure}
\begin{subfigure}[b]{0.3\linewidth}
    \includegraphics[width=\linewidth]{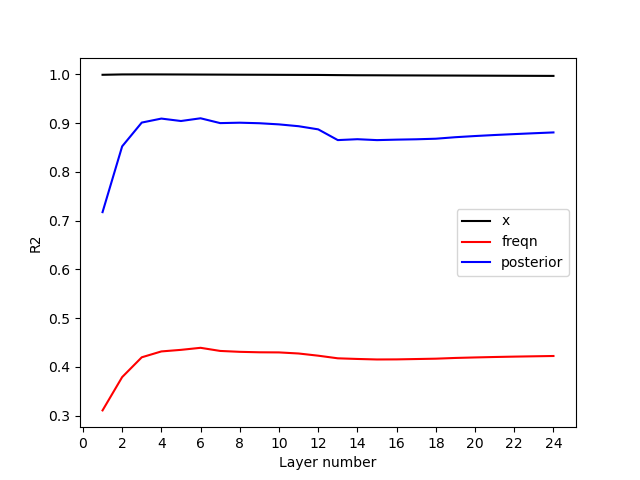}
    \caption{Multinomial prior}
\end{subfigure}
\begin{subfigure}[b]{0.3\linewidth}
    \includegraphics[width=\linewidth]{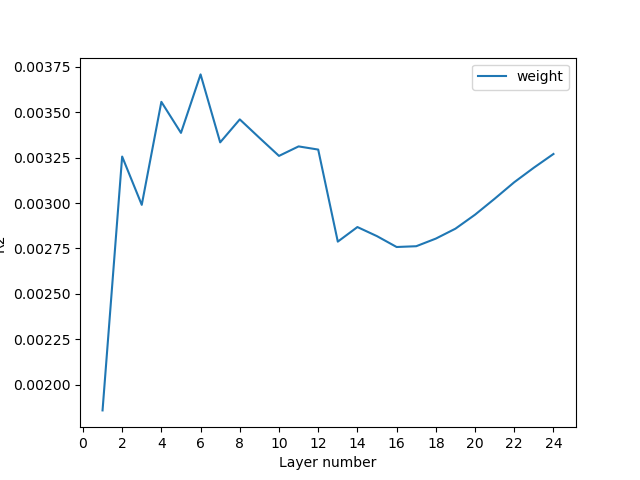}
    \caption{Multinomial prior: \text{PMF}($\theta$)}
\end{subfigure}
\caption{
\textbf{(a), (b):} $R^2$ score of linear probe result against $N(x), f_{\hat{\pi}}(x)$ and $x$ for T24r. We see that while $x$ itself is easily recoverable from any layer, ``knowledge'' about the former two quantities appears to either decrease (in (a)) or plateau (in (b)) with depth. \textbf{(c)} In the multinomial prior case, T24r does not seem to use any information on the atom weight $\text{PMF}_{\pi}(\theta)$. 
}
    \label{fig:linear_probe}
\end{center} 
\end{figure}

\section{Synthetic experiments}\label{sec:synthetic}
We now evaluate our trained transformers on their ability to generalize to sequence lengths and priors unseen during the training stage. 
We also compare against the classical algorithms introduced in \prettyref{sec:task} to demonstrate the superiority of these transformers by showing the average regret; 
most other details are deferred to \prettyref{app:synthetic}. 
We also investigate the inference time to show its advantage over the classical algorithms. 

\subsection{Ability to generalize}\label{sec:generalize}
Our transformer T24r is trained on a mixture of neural and Dirichlet priors, with randomized $\theta_{\max}$, 
and of sequence length 512. We will now describe the aspects we evaluate our trained transformer T24r on, and the corresponding experiments to do so. 
We will also evaluate L24r to investigate the efficacy of linear attention transformers. 

\textbf{Adaptability to various sequence lengths.} 
We evaluate the ability of transformers to adapt to different sequence lengths, 
both fewer than and more than what is trained. 
For each of the experiments, we evaluate our transformers and the baselines on sequence lengths 128, 256, 512, 1024, and 2048. 

\textbf{Robustness against unseen priors}. 
We evaluate our transformers on the following two priors unseen during training: 
(a) worst case prior in $\mathcal{P}([0,50])$ as mentioned in \prettyref{sec:baselines} and further explained in \prettyref{app:worstprior}; 
(b) multinomial prior-on-priors supported on $[0, 50]$ with fixed, evenly split grids and weights following Dirichlet distribution. 
For the worst case prior, the numbers of batches we evaluate on for this prior are 786k (for sequence lengths $n = 128, 256, 512$), 393k (for $n=1024$), and 197k (for $n = 2048$).
For multinomial prior-on-priors, we evaluate on 4096 priors, with the number of batches for each prior set as 4000 ($n = 128, 256, 512$) and 2000 ($n = 1024, 2048$), 
except for NPMLE where we use 192 batches for each prior given the significant computational resource needed. 

\textbf{Robustness under unknown $\theta_{\max}$.} 
We investigate the following: how much can we gain with perfect information on $\theta_{\max}$ during training?
Here, we evaluate our transformer T24r on 4096 neural-generated prior-on-priors with $\theta_{\max} = 50$, 
and compare with another transformer T24f with the same parameters, 
with training prior still neural-Dirichlet mixture (as described in \prettyref{sec:training}) but fixed $\theta_{\max} = 50$. 
The numbers of batches we evaluate on for each of the 4096 priors are 4000 (for sequence lengths $n = 128, 256, 512$) and 2000 (for sequence lengths $n = 1024, 2048$), 
except for NPMLE where we use 192 batches per prior. 
In the plots in \prettyref{fig:seqlen_all} we see that removing information on $\theta_{\max}$ (T24r vs T24f) during training incurs around 40\% extra regret. 

\textbf{Analysis of results}. 
We show our results in \prettyref{fig:seqlen_all}, 
which compares NPMLE, T24r, L24r, and T24f (when applicable). 
Note that we also run against other baselines (ERM-monotone, GS, MLE, Robbins), but their regrets are too big to be shown in the plots. 
There are situations where L24r outperforms T24r, showing that linear transformers perform almost as well. 
However, T24r appears to show a better length generalization property. 

\begin{figure}
    \begin{subfigure}[b]{0.32\linewidth}
        \centering
        \includegraphics[width=\linewidth]{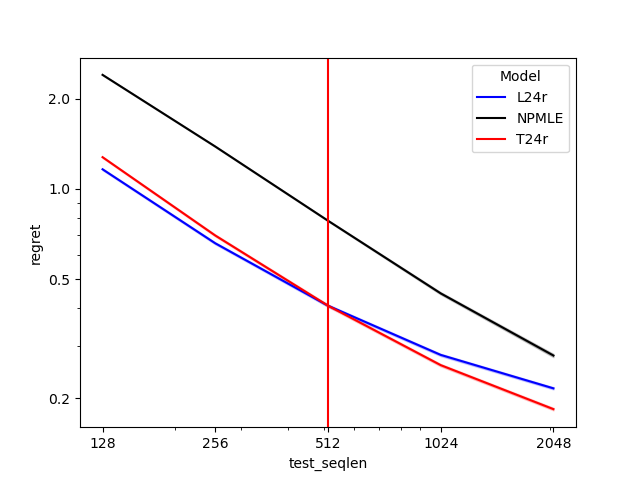}
        \caption{Worst case prior}
        \label{fig:worstprior_seqlen}
    \end{subfigure}
    \begin{subfigure}[b]{0.32\linewidth}
        \centering
        \includegraphics[width=\linewidth]{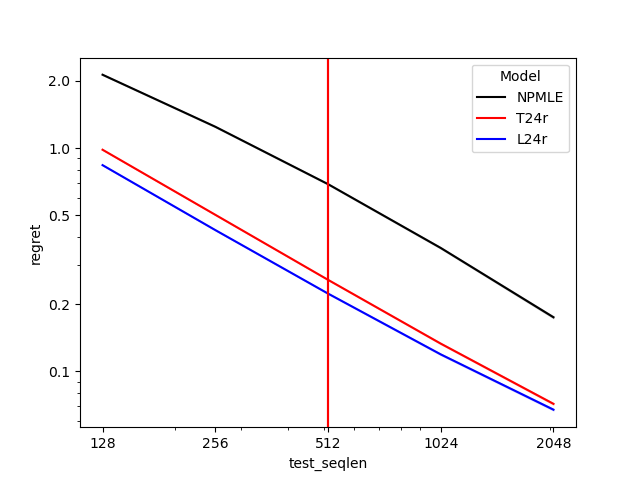}
        \caption{Multinomial prior-on-priors}
        \label{fig:multn_seqlen}
    \end{subfigure}
    \begin{subfigure}[b]{0.32\linewidth}
        \centering
        \includegraphics[width=\linewidth]{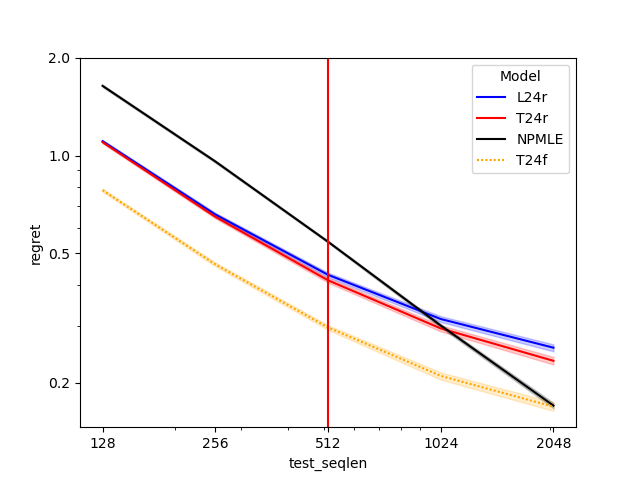}
        \label{fig:regret_seqlen}
        \caption{Neural prior-on-priors}
    \end{subfigure}
\caption{Average regret of NPMLE vs transformers on various priors in $\mathcal{P}([0, 50])$ and on various sequence lengths. The regrets of T24r and L24r decrease with sequence length on all priors and beats NPMLE on most instances. Red vertical lines at 512 denotes the sequence length that the transformers are trained on. 
\textbf{(a)}. T24r outperforms L24r at longer sequence length. ERM-monotone's, MLE's, and Robbin's regrets are 3.26, 11.73, and 85.68 even at sequence length 2048. 
\textbf{(b)}. L24r outperforms T24r throughout, although the gap narrows at $n=2048$. ERM-monotone's, MLE's, GS's, and Robbin's regrets are 2.57, 5.87, 5.63, and 64.43 even at sequence length 2048. 
\textbf{(c)}. NPMLE generalizes better at longer sequence lengths. At $n=4096$ (not shown) NPMLE beats the best performing T24f (regret 0.104 vs 0.153).  
T24r has around 40\% extra regret compared to T24f throughout;
L24r has 10\% extra regret compared to T24r at $n=2048$. 
ERM-monotone's, MLE's, GS's and Robbin's regrets are 2.36, 14.82, 14.66, and 69.40 even at sequence length 2048. 
}
\label{fig:seqlen_all}
\end{figure}

\subsection{Inference time analysis}

\begin{wrapfigure}{r}{0.4\linewidth}
\vskip 0.2in
\includegraphics[width=\linewidth]{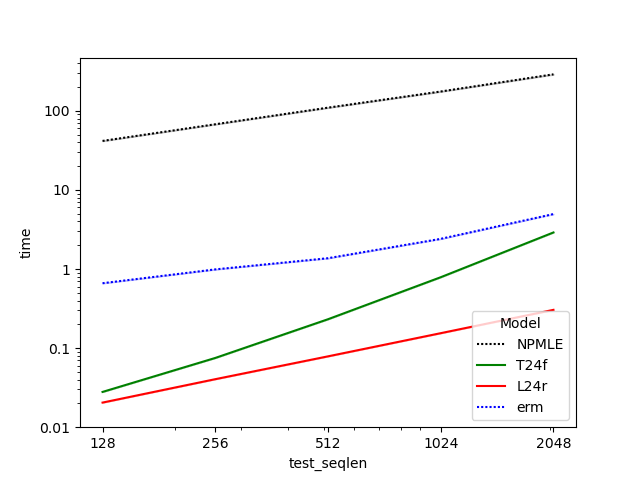}
    \caption{Average time (in seconds) per batch vs sequence length, 
    showing that the inference time of T24r is comparable with that of ERM monotone, 
    and 100x faster than NPMLE. 
    Also shown is L24r which scales better at $n=2048$.
    }
    \label{fig:time_seqlen}
\vskip -0.2in
\end{wrapfigure}

The inference time of T24r in terms of sequence length $n$, number of layers $L$, and embedding dimension is $O(Ln^2d)$. 
In contrast, the time complexity of NPMLE (the most competitive classical algorithm) is difficult to obtain: the algorithms were either based on expectation-maximization (which does not have provable guarantees of convergence) or convex optimization 
(where the rigorous study is also in its infancy \cite{polyanskiy2025nonparametric}). 

To compare T24r against classical EB algorithms, we run our estimators over 4096 neural prior-on-priors, 
where for each prior we consider the time needed to estimate the hidden parameter of 192 batches and sequence lengths $n=$128, 256, 512, 1024, and 2048. 
Each program (transformer and classical baselines) is given 2 Nvidia Volta V100 GPUs and 40 CPUs for computation (340 GB total RAM). 
The results are shown at \prettyref{fig:time_seqlen}, 
where we see that T24r is faster than ERM-monotone on GPU and 100x faster than NPMLE. 
We attribute this to that T24r can better leverage the compute power of GPUs compared to classical algorithms. 
For interest, we will also attach the running time of L24r, which is linear in complexity and runs 10x faster than T24r when $n=2048$. 

\section{Real data experiments}\label{sec:real}

In this section, we answer the following question: 
Can our transformers that are pre-trained on synthetic data perform well on real datasets without re-training on any part of the real datasets, and without any prompt in the form of input-label pairs?

To do so, we consider the following experimental setup: 
Given an integer-valued attribute, let $X$ be the count of the attribute in the initial section we observe, and $Y$ be the count of a similar attribute in the remaining section that we should predict. 
We assume that given a horizon length (duration, sentence length, etc) 
$n_X$ and $n_Y$ of the two sections, there exist hidden parameters $\theta_i$ 
such that $X_i\sim \text{Poi}(n_X\theta_i)$ and $Y_i\sim \text{Poi}(n_Y\theta_i)$, independently 
(for convenience we will scale $\theta_i$ such that $n_X = 1$). 
Our goal is to predict $\hat{Y} = n_Y\hat{\theta}(X)$ using empirical Bayes methods. 
We will focus on the following two types of datasets: 
sports and word frequency. These are some of the main focuses of EB literature 
\cite[Tables 1.1, 11.6]{efron2012large}. 
In the following, we describe the types of datasets that we would study. 
Throughout this section, we name $(X, Y)$ as the input and label sets, respectively. 

We will compare against the classical EB methods. While deep-learning Bayesian models like 
Generative Bayesian Computation \cite{polson2024generative}, TabPFN \cite{hollmann2025accurate}, 
Simformer \cite{gloeckler2024all}, Amortized Conditioning Engine \cite{chang2024amortized} appear to solve empirical Bayes, this is not the case, which we will explain more thoroughly in \prettyref{app:bayesian}. 

\subsection{Sports datasets}\label{sec:sports}
Here, $X$ and $Y$ are the numbers of goals scored by a player within disjoint and consecutive timeframes, 
and $\theta$ represents the innate ability of the given player. We will consider two datasets: 
National Hockey League (NHL) and Major League Baseball (MLB). 

\textbf{NHL dataset}. 
We proceed in the same spirit as \cite[Section 5.2]{jana2022optimal}, 
and study the data on the total number of goals scored by each player in the National Hockey League for 29 years: 
from the 1989-1990 season to the 2018-2019 season (2004-2005 season was canceled). 
The data is obtained from \cite{HockeyReference}, 
and we focus on the skaters' statistics. 
For each season $j$ we predict the goals scored by a player in season $j + 1$ based on their score in season $j$.
(thus the input and label sets are the number of goals a player scored in consecutive seasons, and $n_Y=1$). 
We study the prediction results when fitting all players at once, 
as well as fitting only positions of interest (defender, center, and winger).  

\textbf{MLB dataset}. 
The dataset is publicly available at \cite{Retrosheet}, 
and can be processed by \cite{EstiniRetrosheet}. 
Here, we study the hitting count of each player in batting and pitching players from 1990 to 2017. 
Unlike the between-season prediction as we did for the NHL dataset, we do in-season prediction. 
That is, we take $X$ as the number of goals scored by a player in the beginning portion of the season, 
and $Y$ in the rest of the season. 
For both batting and pitching players (which we fit separately), 
we use $X$ as the goals in the first half of the season (i.e. $n_Y = 1$). 

\subsection{Word frequency datasets}\label{sec:word-freq}
In this setting, we model the alphabet of tokens as $M$ categorical objects $A = \{A_1, \cdots, A_M\}$. 
Given $n$ samples from these objects, we denote $(X_1, \cdots, X_M)$ the frequency of the samples. 
We are to estimate the frequencies $(Y_1, \cdots, Y_M)$ of an unseen section of length $t$ (here $t$ known). 
We model as follows: 
consider $p_1, \cdots, p_M$ as the ``inherent'' probability distribution over $M$ 
(or proportion in a population), so $\sum_{i = 1}^M p_i = 1$. 
Now the frequency $X_i \sim \text{Binom}(n, p_i)$, which we may instead approximate as $X_i\sim\text{Poi}(np_i)$. 
Thus we may use empirical Bayes method to estimate $\hat{\theta}_i = n\hat{p}_i$ based on the frequencies $X_1, \cdots, X_M$, 
and then predict $\hat{Y}_i = \frac{t}{n}\hat{\theta}_i$. 

\textbf{BookCorpusOpen}. 
BookCorpus is a well-known large-scale text dataset, 
originally collected and analyzed by \cite{zhu2015aligning}. 
Here, we use a newer version named BookCorpusOpen, hosted on websites like 
\cite{BookCorpusOpen}. 
This version of the dataset contains 17868 books in English; 
we discard 6 of the books that are too short ($\le 2000$ tokens), 
and 5 other books where NPMLE incurs out-of-memory error. 
To curate the dataset, we first tokenize the text using scikit-learn's 
CountVectorizer with English stopwords removed. 
For each book, the input set comprises the beginning section containing approximately 2000 tokens, 
while the label set the remainder of the book. 
Then for each word, $X$ and $Y$ are the frequency of each word within the input and label set, respectively. 
We will then use the prediction $\hat{Y} = n_Y\cdot \hat{\theta}(X)$ where $n_Y$ is the ratio of the number of sentences in the label set to that of the input set. 

\subsection{Evaluation methods}
We will use the RMSE of each dataset item, normalized by $n_Y$, as our main evaluation metric. 
Specifically, for each dataset, we compute the RMSE incurred by each estimator. 
We then compare them using the following guidelines. 

\textbf{Comparison against MLE}. 
We consider the ratio of RMSE of each estimator against that of the MLE, and ask, ``how much improvement did we achieve against the MLE'' by looking at the \emph{average} of the ratio. 

\textbf{Relative ranking}. We use the Plackett-Luce \cite{plackett1975analysis,luce1959individual} ranking system to determine how well one estimator ranks over the other. 

\textbf{Significance of improvement}. We consider whether one improvement is \emph{significant} by performing paired $t$-test on the RMSE of transformers against the baselines and report the $p$-value. 

We will only show percentage improvement over MLE in \prettyref{tab:percentage_rmse}; 
$t$-tests and Plackett-Luce ranking are deferred to \prettyref{app:real}.  
We also show violin plots in \prettyref{fig:violinplots_realdatasets} to supplement our findings 
(with Robbins removed due to its wide variance). 
A more detailed comparison, including the use of the MAE metric, is shown in \prettyref{app:real}. 
These experiments seem to support that transformers, including the much faster L24r, can outperform the classical EB estimators in real-life settings. 

\begin{figure}
    \begin{subfigure}[b]{0.24\linewidth}
        \centering
        \includegraphics[width=\linewidth]{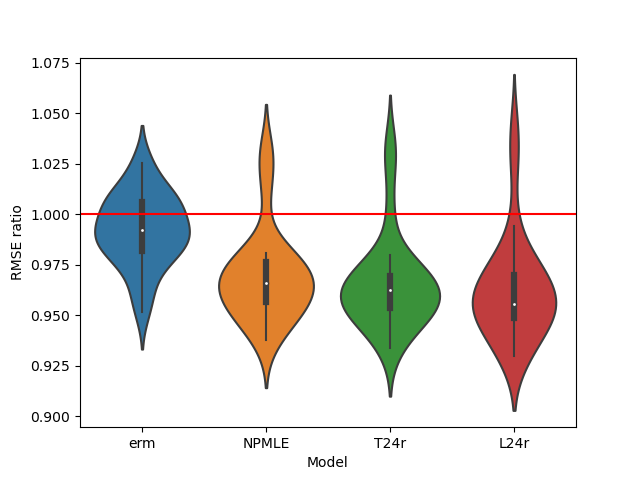}
        \caption{NHL Hockey}
        \label{fig:hockey_violin}
    \end{subfigure}
    \begin{subfigure}[b]{0.24\linewidth}
        \centering
        \includegraphics[width=\linewidth]{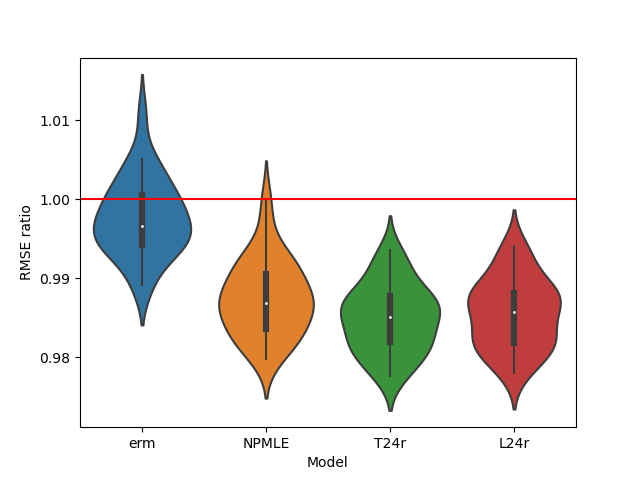}
    \caption{MLB (Batting)}
    \label{fig:batting_violin}
    \end{subfigure}
    \begin{subfigure}[b]{0.24\linewidth}
        \centering
        \includegraphics[width=\linewidth]{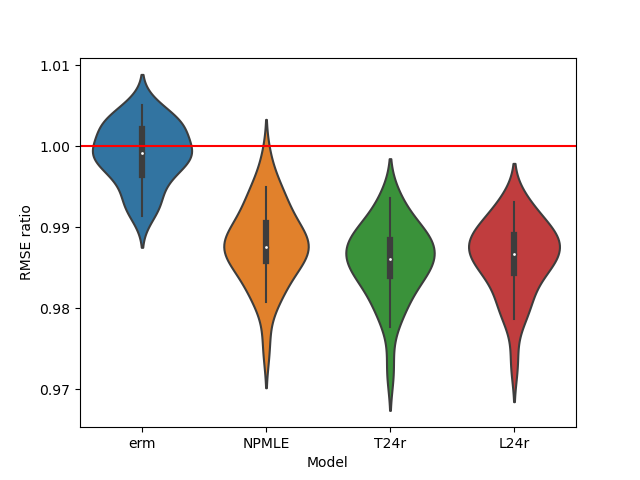}
    \caption{MLB (Batting)}
    \label{fig:pitching_violin}
    \end{subfigure}
    \begin{subfigure}[b]{0.24\linewidth}
        \centering
        \includegraphics[width=\linewidth]{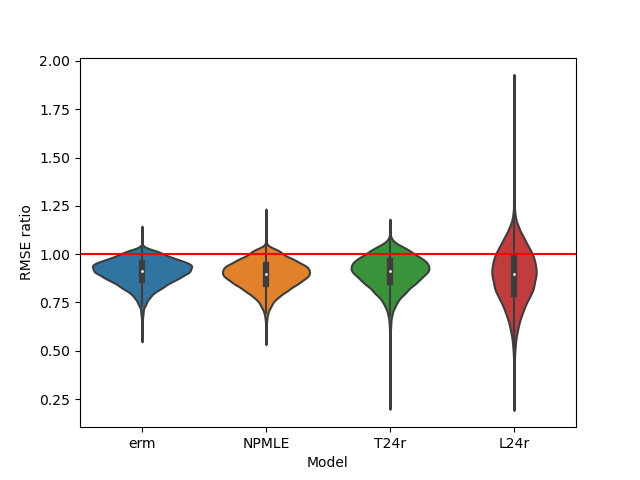}
    \caption{BookCorpusOpen}
    \label{fig:bookcorpus_violin}
    \end{subfigure}
\caption{Violin plots of RMSE ratio of ERM-monotone (blue), NPMLE (orange), T24r (green), and L24r (red) over MLE over various datasets. The horizontal line at 1.0 indicates the threshold of MLE's RMSE. 
Robbins is not shown due to the wide variance. 
Both T24r and L24r show a general improvement over MLE, and in the case of the MLB dataset their advantage over NPMLE is visible. 
}
\label{fig:violinplots_realdatasets}
\end{figure}

\begin{table*}[ht]
\caption{95\% confidence interval of the percentage improvement of RMSE over MLE.}
\label{tab:percentage_rmse}
\vskip 0.15in
\begin{center}
\begin{small}
\begin{sc}
\begin{tabular}{lccccc}
\toprule
Dataset & Robbins & ERM & NPMLE & T24r & L24r \\
\midrule
NHL (all) & -30.55 $\pm$ 6.55 & 1.46 $\pm$ 0.65 & 3.21 $\pm$ 0.92 & 3.46 $\pm$ 0.88 & \textbf{3.96 $\pm$ 1.12} \\
NHL (defender) & -19.54 $\pm$ 6.35 & 3.19 $\pm$ 1.32 & 6.48 $\pm$ 1.63 & 6.91 $\pm$ 1.71 & \textbf{7.54 $\pm$ 2.04} \\
NHL (center) & -49.89 $\pm$ 10.36 & 0.38 $\pm$ 0.82 & 3.44 $\pm$ 0.94 & 4.02 $\pm$ 0.99 & \textbf{4.32 $\pm$ 1.22} \\
NHL (winger) & -42.63 $\pm$ 7.58 & 0.76 $\pm$ 0.69 & 3.06 $\pm$ 0.87 & 3.44 $\pm$ 0.89 & \textbf{3.76 $\pm$ 0.99} \\
\midrule
MLB (batting) & -59.66 $\pm$ 5.88 & 0.23 $\pm$ 0.18 & 1.25 $\pm$ 0.18 & \textbf{1.50 $\pm$ 0.16} & 1.44 $\pm$ 0.17 \\
MLB (pitching) & -40.81 $\pm$ 3.16 & 0.09 $ \pm$ 0.14 & 1.21 $ \pm$ 0.19 & \textbf{1.42 $\pm$ 0.18} & 1.38 $ \pm$ 0.17 \\
\midrule
BookCorpusOpen  & -4.58 $\pm$ 0.43 & 9.38 $\pm$ 0.10 & 10.82 $\pm$ 0.11 & 9.43 $\pm$ 0.12 & \textbf{11.23 $\pm$ 0.21} \\
\bottomrule
\end{tabular}
\end{sc}
\end{small}
\end{center}
\vskip -0.1in
\end{table*}

\section{Conclusion, limitation and future work}\label{sec:conclusion}
We have demonstrated the ability of transformers to learn EB-Poisson via in-context learning. 
This was done by evaluating pre-trained transformers on synthetic data of unseen distribution and sequence lengths, 
and compared against baselines like the NPMLE and other EB estimators. 
In this process, we showed that transformers can achieve decreasing regret as the sequence length increases. 
On the real datasets, we showed that these pre-trained transformers can outperform classical baselines in most cases. 

\textbf{Limitation}. 
We humbly remark that we limited our focus to the 1-dimensional Poisson problem. Therefore, one future direction will be to extend our work to multi-dimensional input 
\cite[Section 1.3]{jana2023empirical}, \cite[Section 6]{jana2022optimal}, 
and to other models like the normal-means model \cite{jiang2009general}. 
We focused on the mean estimation problem, while a full scale posterior inference involves tasks like posterior sampling and uncertainty quantification. 
We believe that (probably larger) transformers could also achieve such tasks in the empirical Bayes setting. 
All these extensions, we believe, will expand the set of real-world applications, 
and can better demonstrate the power (or limitation) of transformers. 
While we established tools (Robbins and NPMLE approximation, and linear probe) to understand how transformers solve mean estimation in EB-Poisson setting, there are other aspects we have not discussed (e.g. training dynamics). 
These will be subject to future work.

  \section*{Acknowledgements} 
  This work was supported in part by the MIT-IBM Watson AI Lab and the National Science Foundation under Grant No CCF-2131115. Anzo Teh was supported by a fellowship from the Eric and Wendy Schmidt Center at the Broad Institute. The authors acknowledge the MIT SuperCloud and Lincoln Laboratory Supercomputing Center for providing compute resources that have contributed to the research results reported within this paper.

\bibliographystyle{alpha} 
\bibliography{references}

\newpage

\appendix
\section{Detailed Discussion on Setups}

\subsection{Worst-case prior and Gold-Standard Estimator}\label{app:worstprior}
We first define the worst-case prior and the gold-standard estimator. 
\begin{definition}[Worst-case prior]
    Let $A$ be a compact subset of $\mathbb{R}$. 
    Then the worst-case prior $\pi_{!, A}$ is defined as
    \[
    \pi_{!, A} = \argmax_{\pi\in \mathcal{A}} \mathsf{mmse}(\pi)
    \]
\end{definition}

A sample distribution of the worst-case prior on $[0, 50]$ is illustrated in 
Figure 1 of \cite{jana2022optimal}. 

One motivation for using the worst-case prior is that the Bayes estimator is considered the ``gold standard'' which minimizes the maximum possible MSE across all priors supported on $A$. 
A concrete statement can be found in the following lemma. 
\begin{lemma}\label{lmm:worst_prior_mse}
    Let $\hat{\theta}_{\pi}$ be the Bayes estimator of a given prior $\pi$, and let $A$ be any compact subset of the reals. 
    Then the least favorable prior $\pi_{!, A}$ of $A$ satisfies the following: 
    \[
    \mathsf{MSE}_{\delta_{\theta}}(\hat{\theta}_{\pi_{!, A}}) \le \mathsf{mmse}(\pi_{!, A}), \forall \theta\in A
    \]
    and equality holds whenever $\theta\in\mathsf{Supp}(\pi_{!, A})$. 
\end{lemma}

This leads to the following corollary. 
\begin{corollary}\label{cor:gs-minmaxmmse}
    For any compact subset $A$ of the real numbers, we have 
    \[
    \min_{\hat{\theta}}\max_{\pi\in\mathcal{P}(A)} \mathbb{E}[(\hat{\theta}(X) - \theta)^2]
    =\mathsf{mmse}(\pi_{!, A})
    \]
    achieved by the Bayes estimator $\hat{\theta}_{\pi_{!, A}}$ of the least favourable prior, $\pi_{!, A}$. 
\end{corollary}

\begin{proof}
From \prettyref{lmm:worst_prior_mse}, we have $\mathsf{MSE}_{\pi}(\hat{\theta}_{\pi_{!, A}})\le \mathsf{mmse}(\pi_!)$ for any $\pi\in\mathcal{P}(A)$. 
Therefore $\min_{\hat{\theta}}\max_{\pi\in\mathcal{P}(A)} \mathbb{E}[(\hat{\theta}(X) - \theta)^2]\le\mathsf{mmse}(\pi_!)$ by taking 
$\hat{\theta} = \hat{\theta}_{\pi_{!, A}}$. 
Now, for any $\hat{\theta}$, we have $\mathbb{E}_{\pi_{!, A}}[(\hat{\theta}(X) - \theta)^2]\ge \mathsf{mmse}(\pi_!)$. 
Therefore the conclusion follows. 
\end{proof}

\prettyref{cor:gs-minmaxmmse} says that among all the non-empirical Bayes estimators, $\hat{\theta}_{\pi_{!, A}}$ achieves the best minimax MSE loss among all priors supported on $A$. 
On the flip side, however, this estimator $\hat{\theta}_{\pi_{!, A}}$ does not leverage the fact the low-MMSE nature of some prior, leading to suboptimal regret produced by $\hat{\theta}_{\pi_{!, A}}$. 
In \prettyref{fig:mmse_neural}, we display the priors generated by the neural prior-on-priors 
and the histogram of MMSEs. 
The MSE given by $\hat{\theta}_{\pi_{!, [0, 50]}}$ on priors that are point masses as per 
\prettyref{fig:worst_prior_bayes_mse} suggests that $f_{\pi!, [0, 50]}$ is incapable of achieving low regrets on priors with low MMSEs. 

\begin{figure}[htbp]
    \begin{subfigure}[b]{0.5\linewidth}
        \centering 
        \includegraphics[width=\linewidth]{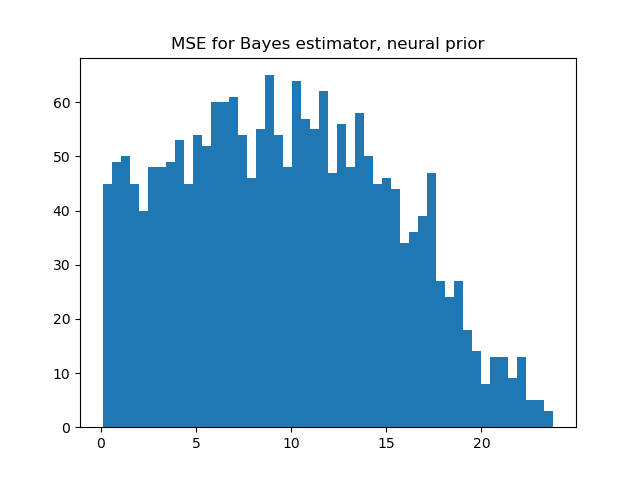}
        \caption{MMSE of neural priors in $\mathcal{P}([0, 50])$}
        \label{fig:mmse_neural}
    \end{subfigure}
    \begin{subfigure}[b]{0.5\linewidth}
        \centering 
        \includegraphics[width=\linewidth]{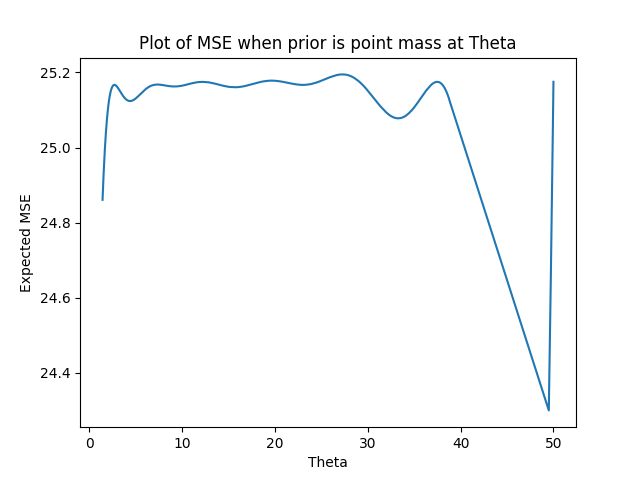}
        \caption{MSE of $f_{\pi!, [0, 50]}$ at point masses}
        \label{fig:worst_prior_bayes_mse}
    \end{subfigure}     
    \caption{Discussion on Worst Prior}
\end{figure}

\subsection{Training priors}\label{app:train-priors}
We now offer a more detailed description of the training priors. 
Recall in \prettyref{sec:training} we defined $\theta = \theta_{\mathsf{base}}\cdot \theta_{\max}$ with $\theta_{\mathsf{base}}\in [0, 1]$. Here, we explain how $\theta_{\mathsf{base}}$ is generated according to both the neural-generated prior-on-priors, 
and the dirichlet process based on the uniform distribution. 
Recall that $\theta_{\max}$ is randomized at training phase according to how we described in \prettyref{sec:training}, 
but nevertheless fixed during evaluation phase. 

\textbf{Neural-generated: prior-on-priors.} 
We sample $\theta_{\mathsf{base}}\in [0, 1]$ via the following: first, let $\mathcal{M}$ be classes of priors determined by some two-layer perceptron with a non-linear activation in-between. 
This is concretely defined as: 
\begin{equation*}
    \mathcal{M} = \{\pi: \pi = \varphi^{W_1, W_2, \sigma}_{\sharp} \mathsf{Unif}[0, 1]\}
\end{equation*}
where $\varphi^{W_1, W_2, \sigma}(x) = \mathsf{Sigmoid}(10W_2\sigma(W_1x))$, $W_1, W_2$ are linear operators, and $\sigma$ is an activation function chosen randomly from 
\[\{GELU, ReLU, SELU, CELU, SiLU, Tanh, TanhShrink\}. \]
$\theta_{\mathsf{base}}$ is then produced by sampling from a mixture of 4 priors in $\mathcal{M}$. 

\textbf{Dirichlet process}. 
Let the base distribution be defined as $H_0\triangleq\text{Unif}([0, 1])$. 
Within each batch, elements $\theta_{\mathsf{base}, 1}, \cdots, \theta_{\mathsf{base}, n}$ are generated as follows: 
\[\theta_{\mathsf{base}, j} = 
    \begin{cases}
        \theta_{\mathsf{base}, i} & \text{w.p. }\frac{j - 1}{\alpha + j - 1},\forall i = 1, \cdots, j - 1\\
        \theta_{\mathsf{base}}\sim H_0 & \text{w.p. }\frac{\alpha}{\alpha + j - 1}
    \end{cases}
\]
where $\alpha$ is a parameter that denotes how `close' we are to i.i.d. generation 
($\alpha=\infty$ essentially means we have iid). 
We use $\alpha = 50$ throughout (recall that the sequence length is 512). 
Note that Dirichlet process implies that our data is not generated i.i.d. for each batch, 
so the Bayes estimator has to be estimated differently. We omit the calculation of this Bayes estimator. 

\textbf{Mixture of two priors.} 
We choose $\theta_{\mathsf{base}}$ from each of the two classes of prior with probability $\frac 12$, 
but each batch contains $\theta$'s drawn only from the same prior. 
This allows us to train our transformers to learn information about the prior $\pi$ that generates $\theta$ upon seeing sequence $(X_1, X_2, \cdots, X_n)$. 

\subsection{Why do we train using a mixture of two prior classes?}
We consider the hypothesis: that our transformer trained under the mixture of the two priors is robust when evaluated under each of the priors. 
This can be verified via the following two tests: 
when evaluated on neural prior, is the performance (in terms of MSE) of the mixture-trained transformers closer to that of neural-trained ones as compared to the Dirichlet-trained ones? 
Similarly, when evaluated on Dirichlet prior, 
is the performance (in terms of MSE) of the mixture-trained transformers closer to that of Dirichlet-trained ones as compared to the neural-trained ones? 
By compring the MSEs of 4096 seeds at evaluation stage at \prettyref{tab:mixture-training}, 
we answer both these questions in the positive 
(the difference is especially obvious when evaluated on neural prior). 

\begin{table}[ht]
\centering
\caption{Table of regret difference; $A - B$ denotes the difference of regret of transformers trained on $A$ vs trained on $B$}
\label{tab:mixture-training}
\begin{tabular}{ccccc}
\hline
 & \multicolumn{2}{c}{Evaluated on Neural}  &  \multicolumn{2}{c}{Evaluated on Dirichlet} \\ 
\hline 
\hline 
\# lyr & mix $-$ neu & dir $-$ mix & mix $-$ dir & neu $-$ mix\\
\hline 
12 & 0.0038 & 0.8645 & 0.0184 & 0.0379\\
18 & 0.0133 & 1.0647 & 0.0173 & 0.0469\\
24 & 0.0082 & 1.0021 & 0.0202 & 0.0388\\
\hline 
\end{tabular}
\end{table}

\section{Technical Proofs}\label{app:proofs}
\subsection{Approximation of known empirical Bayes baselines}
\begin{proof}[Proof of \prettyref{thm:robbins-transformers}]
    \textbf{Encoding step.}
    We embed our inputs representation $X\in\mathbb{R}^n$ into one-hot vector $Y\in\mathbb{R}^{n\times (d + 1)}$ such that 
    $Y_i = e_{X_i + 1}$ if $X_i = 0, 1, \cdots, d$, and 0 otherwise. 
    Then given sample size $n$, $Y \in\mathbb{R}^{(d + 1)\times n}$. 
    Now recall the following attention layer definition in 
    (1) of \cite{vaswani2017attention}: 
    \[
    \text{Attention}(Q, K, V) = \text{softmax}\left(\frac{QK^T}{\sqrt{d_k}}\right) V
    \]
    where $Q=YW_Q, K = YW_K, V = YW_V$, and $W_Q, W_K\in\mathbb{R}^{(d + 1)\times (d+ 1)}$. 
    Let $Z = \text{Attention}(Q, K, V)$. 
    We now design an encoding mechanism such that the representation after skip connection has the following: 
    \[
    (Y + Z)_{ij} = 
    \begin{cases}
        1 + \frac{N(X_i)}{N(X_i) + (X_i + 1)N(X_i + 1)} & j = X_i + 1\le d\\
        \frac{(X_i + 1)N(X_i + 1)}{N(X_i) + (X_i + 1)N(X_i + 1)} & j = X_i + 2\le d\\
        1 & j = X_i + 1 = d\\
        0 & \text{otherwise}.\\
    \end{cases}
    \]

    Define $D$ be a large number, 
    $W_Q = I_{d + 1}$ the $d$-dimensional identity matrix, 
    $W_V = \begin{pmatrix}
        I_d & 0\\ 
        0 & 0\\
    \end{pmatrix}$
    and $W_K\in\mathbb{R}^{(d + 1)\times (d + 1)}$ satisfying
    \[
    (W_K)_{i, j} = 
    \begin{cases}
        D & i = j\\
        D + \sqrt{d + 1}\log i & j = i + 1\\
        0 & \text{otherwise}\\
    \end{cases}
    \] (thus $d_k = d + 1$). Then 
    \[
    (QK^T)_{i, j} = 
    \begin{cases}
        D & X_i = X_j \le d\\
        D + \sqrt{k}\log(X_i + 1) & X_j = X_i + 1\le d\\
        0 & \text{otherwise}\\
    \end{cases}. 
    \]
    Thus we have the following structure for $M \triangleq \text{Softmax}(S)$: 
    $\text{row}_i(M) = \frac 1n$ if $X_i\ge d + 1$, 
    otherwise \[M_{ij} = 
    \begin{cases}
        \frac{1}{N(X_i) + (X_i + 1)N(X_i + 1)} & X_i = X_j\le d - 1\\
        \frac{X_i + 1}{N(X_i) + (X_i + 1)N(X_i + 1)} & X_j = X_i + 1\le d\\
        0 & \text{otherwise}.\\
    \end{cases}\]
    Now given that $V = 
    \begin{pmatrix}\text{Col}_1(Y) & \cdots & \text{Col}_d(Y) & 0\end{pmatrix}$, $Z_{ij} = \sum_{k: j = X_k + 1} Z_{ik}$ for all $k\le d - 1$ (and 0 for $k$)
    This means: 
    \[Z_{ij} = 
    \begin{cases}
        \frac{N(X_i)}{N(X_i) + (X_i + 1)N(X_i + 1)} & j = X_i + 1\le d - 1\\
        \frac{(X_i + 1)N(X_i + 1)}{N(X_i) + (X_i + 1)N(X_i + 1)} & j = X_i + 2\le d\\
        0 & \text{otherwise}.\\
    \end{cases}\]
    Thus, adding back $Y$ gives the desired output. 

    \textbf{Decoding step.}
    We define $Y_1 = \mathsf{ReLU}(Y + Z- 1)$, i.e. a linear operation (with bias) followed by the ReLU nonlinear operator. Notice that $Z$ has entries all in $[0, 1]$, 
    so $Y_1$ acts like $Y* Z$ (i.e. $Z$ masked with $Y$). 
    Let $Z' \in \mathbb{R}^n$ to be the row-wise sum of $Y_1$, 
    i.e. 
    $Z'_i = \frac{N(X_i)}{N(X_i) + (X_i + 1)N(X_i + 1)}$ if $X_i\le d - 1$ and 0 otherwise. 
    Then we consider the following decoding function $f: [0, 1]\to [0, \theta_{\max}]$ by: 
    \[
    f(x) = 
    \begin{cases}
        \frac{1}{x} - 1 & x\ge \frac{1}{1 + M}\\
        M & \text{otherwise}\\
    \end{cases}. 
    \]
    Then $f$ is continuous, and $f(Z')$ is indeed $\hat{\theta}_{\mathsf{Rob}, d, M}$. 
    In addition, this $f$ is $(1+M)^2$-Lipchitz, 
    and therefore by the results in \cite[Theorem 2.1]{mhaskar1996neural}, 
    there exists a two-layered network with $O((1+M)^2\epsilon^{-1})$ hidden neurons that approximates $f$ within $[0, 1]$ with error at most $\epsilon$, as desired. 
\end{proof}

\begin{remark}
We note that Robbins approximation using linear attention is also possible, using the attention defined in \cite[(4)]{katharopoulos2020transformers}. 
Specifically, 
We embed our inputs representation $X\in\mathbb{R}^n$ into one-hot vector $Y\in\mathbb{R}^{n\times (d + 1)}$
where $Y_i = e_{X_i+1}$ if $X_i < d$ and $e_{d+1}$ if $X_i\ge d$ (instead of 0 to prevent the denominator in \cite[(4)]{katharopoulos2020transformers} for normalization to be 0). 
We choose $\phi$ the identity function,
    $W_Q = I_{d + 1}$ the $d + 1$-dimensional identity matrix, 
    $W_V = \begin{pmatrix}
        I_{d} & 0\\ 
        0 & 0\\
    \end{pmatrix}$, and $W_K$ the following tridiagonal matrix: 
    \[
    (W_K)_{ij} = 
    \begin{cases}
        1 & i = j\\
        j & j = 1\le i - 1\le d - 1\\
        0 & \text{otherwise}\\
    \end{cases}
    \]
    This results in the following: 
    \[
    Q_i = 
    \begin{cases}
        e_{X_i + 1} & X_i < d\\
        e_{d+1} & \text{otherwise}\\
    \end{cases}
    \qquad 
    K_j = 
    \begin{cases}
        e_{X_j + 1} + X_je_{X_j} & X_j < d\\
        e_{d+1} & \text{otherwise}\\
    \end{cases}
    \qquad 
    V_j = 
    \begin{cases}
        e_{X_j + 1} & X_j < d\\
        0 & \text{otherwise}\\
    \end{cases}
    \]
    and will therefore yield the same attention output after normalization. 
\end{remark}

Before proving \prettyref{thm:univ_npmle}, we need to establish the continuity of the clipped NPMLE, 
with arguments the sigmoid of the input integer and empirical distribution. 
\begin{lemma}\label{lmm:npmle_cont}
    Let $\varphi: \mathbb{R}_{\ge 0} \to [0, 1]$ be a strictly increasing and continuous function, 
    and $\text{Sig} = \{\varphi(z): z\in \mathbb{Z}\}$. 
    Let $S =  \sup(\text{Sig})$ and $\text{Sig}^+ = Sig \cup \{S\}$. 
    Denote $\tilde{\theta}: (\mathcal{P}(\text{Sig}^+)\times \text{Sig}^+\to [0, M]$ be such that for each $p^{\mathsf{emp}}\in \mathcal{P}(\mathbb{Z}_{\ge 0})$ and 
    $x\in \mathbb{Z}_{\ge 0}$, 
    the function $\tilde{\theta}(\varphi_{\sharp}(p^{\mathsf{emp}}), \varphi(x)) = \hat{\theta}_{\mathsf{NPMLE}, d, M}(p^{\mathsf{emp}}, x)$. 
    Then $\tilde{\theta}$ can be extended into a function that is continuous in both arguments. 
    (Here $p^{\mathsf{emp}}$ acts like an empirical distribution). 
\end{lemma}

\begin{proof}[Proof of \prettyref{thm:univ_npmle}]
    Starting with the input tokens $(X_1, \cdots, X_n)$, 
    we consider the token-wise embedding 
    $Y_i = \mathsf{Sigmoid}(X_i)$. 
    Note that $\mathsf{Sigmoid}$ satisfies the assumption of $\varphi$ in \prettyref{lmm:npmle_cont}. 
    Denote $p_n^{\mathsf{emp}}$ as the empirical distribution determined by $(X_1, \cdots, X_n)$. 
    By \prettyref{lmm:npmle_cont}, 
    the function 
    $\hat{\theta}_{\mathsf{NPMLE}, d}(p_n^{\mathsf{emp}}, \cdot)$ can be continuously extended 
    (in the (weak$^*$, $\ell_2$) metric). Then \cite{furuya2024transformers}, Theorem 1 says that there exists a transformers network $\Gamma$ that satisfies 
    \[
    |\hat{\theta}(x_1, \cdots, x_n)_i - \Gamma(x_1, \cdots, x_n)_i|\le \epsilon\,,
    \]
    as desired. 
\end{proof}

\begin{proof}[Proof of \prettyref{lmm:npmle_cont}]
    To establish continuity, it suffices to show that given a sequence of distributions
    $p^{\mathsf{emp}}_1, p^{\mathsf{emp}}_2, \cdots$ and integers $x_1, x_2, \cdots$ such that 
    $\varphi_{\sharp}(p^{\mathsf{emp}}_n)\to \phi_0$ and $\varphi(x_n)\to y_0$ 
    in (weak$^*$, $\ell_2$) metric, we have 
    $\hat{\theta}_{\mathsf{NPMLE}}(p^{\mathsf{emp}}_n, x_n) \to  \tilde{\theta}(\phi_0, y_0)$. 
    Note that $x_n$ are nonnegative integers, 
    so given that $\varphi$ is increasing and injective, 
    either $x_n$ is eventually constant (in which case $x_n\to x_0$ for some $x_0$), or $x_n\to \infty$ 
    (in which case $y_0 = S \triangleq\sup(\text{Sig})$). 
    Note first that in the case $y_0 = S$ we have 
    $\tilde{\theta}(\varphi_{\sharp}(p^{\mathsf{emp}}_n), \varphi(x_n)) = M$ for all $n$ sufficiently large. 
    Note also that $p^{\mathsf{emp}}$ is a distribution on nonnegative integers so 
    $\varphi_{\sharp}(p^{\mathsf{emp}}_n)(S) = 0$, 
    which then follows that $\phi_0(S) = 0$ too. 
    Thus $\phi_0\in \mathcal{P}(\text{Sig})$ and so there exists $p_0$ such that 
    $\phi_0 = \varphi_{\sharp}(p_0)$. 

    It now remains to consider the case where $x_n = x_0$ for all sufficiently large $n$; 
    w.l.o.g. we may even assume $x_n = x_0$ for all $n$. 
    If $x_0 > d$ we are done since $\hat{\theta}_{\mathsf{NPMLE}}(p^{\mathsf{emp}}, x_0) = M$ for all $\pi$.  
    Assume now that $x_0 \le d$. 
    Recall that $\hat{\theta}_{\mathsf{NPMLE}}(p^{\mathsf{emp}}, x_0) = (x_0 + 1) \frac{f_{\hat{\pi}}(x_0 + 1)}{f_{\hat{\pi}}(x_0)}$ where $\hat{\pi}$ is the prior estimated by NPMLE. 
    Thus denoting $\hat{\pi}_n$ as NPMLE prior of $p^{\mathsf{emp}}_n$, for each $x$ it suffices to show that convergence of $f_{\hat{\pi}_n}(x)$.  
    Now note that NPMLE also has the following equivalent form: 
    $\hat{\pi}_n
    = \argmin_Q D(p^{\mathsf{emp}}_n || f_Q)$, 
    where $D$ denotes the KL divergence. 
    Note that $D$ can be written in the following form 
    (c.f. 
    Assumption 1 of \cite{jana2022optimal}). 
    \begin{equation}\label{eq:mindist_eqn}
    D(\pi_1 || \pi_2) = t(\pi_1)
    + \sum_{x\ge 0} \ell(\pi_1, \pi_2)
    \end{equation}
    Notice that $\ell(a, b) := a\log \frac{1}{b}$ fulfills 
    $b\to \ell(a, b)$ is strictly decreasing and strictly convex for $a > 0$. 
    Fix $x_0\le d + 1$, we now have two subcases: 
    
    \textbf{Case $p_0(x_0) > 0$.}
    The claim immediately follows from that $\ell(p_0(x_0), b)$ is \emph{strictly} convex in $b$, 
        and that for each $x$ we have $p^{\mathsf{emp}}_n(x) \to p_0(x)$ following the weak convergence of $p^{\mathsf{emp}}_n$. 

    \textbf{Case $p_0(x_0) = 0$.} 
    Let $Q_0\in \argmin_Q t(p_0) + \sum_{x\ge 0} \ell(p_0(x), f_Q(x))$. 
    By Theorem 1 of \cite{jana2022optimal}, 
    this $Q_0$ is unique. 
    Now suppose that there is a subsequence 
    $n_1, n_2, \cdots$ and a real number $\epsilon > 0$ such that 
    \[
    |f_{\hat{\pi}_{n_i}}(x_0) - f_{Q_0}(x_0)| > \epsilon
    \]
    By the previous subcase, we have 
    $f_{\hat{\pi}_{n_i}}(x)\to f_{Q_0}(x)$ for all $x\in \text{Supp}(p_0)$. 
    We now consider $Q_1$ as the solution to 
    \prettyref{eq:mindist_eqn}, but among the class of functions satisfying the constraint 
    $|f_{Q_1}(x) - f_{Q_0}(x)| > \epsilon$. 
    Such a constrained space is closed by proof of 
    Theorem 1 in \cite{jana2022optimal}, 
    so there exists $\delta > 0$ such that 
    \begin{align*}
        ~D(p_0 || f_{\hat{\pi}_{n_i}}) - D(p_0 || f_{Q_0})
        \ge D(p_0 || f_{Q_1}) - D(p_0 || f_{Q_0}) \ge \delta
    \end{align*}
    On the other hand, by fixing $Q, Q'$, we have 
    \begin{align*}
    (D(p_n^{\mathsf{emp}} || f_{Q}) - D(p_n^{\mathsf{emp}} || f_{Q'}))
    -  (D(p_0 || f_{Q}) - D(p_0 || f_{Q'}))
    \to 0
    \end{align*}
    given that $p_n^{\text{emp}}\to p_0$ weakly 
    and that 
    $D(p || f_{Q}) - D(p || f_{Q'})= \sum_y p(y)\log \frac{f_Q'}{f_Q}$, 
    which is a contradiction. 
\end{proof}

\begin{proof}[Proof of \prettyref{cor:vanish-regret}]
    Choose $d$ such that $\mathbb{P}[X > d] < \frac{\epsilon}{6\cdot \theta_{\max}^2}$. 
    By the properties of the Poisson distribution (e.g. \cite[(15.19)]{polyanskiy2025information} on Poisson-Binomial distribution), 
    for each $x > \theta_{\max}$ we have 
    \[
    \mathbb{P}[X\ge x] \le \sup_{0\le \theta\le\theta_{\max}}\frac{(e\theta)^x e^{-\theta}}{x^x}
         \le \frac{(e\theta_{\max})^x e^{-\theta_{\max}}}{x^x}
    \]
    Thus there exists some $d\in O(\log (\frac{1}{\epsilon}) + \log \theta_{\max})$ that fulfills this condition. 

    Next, with sequence length $n$, both the Robbins estimator \cite[Theorem 2]{polyanskiy2021sharp} and NPMLE \cite[Theorem 3]{jana2022optimal},enjoy a minimax regret of $O_{\theta_{\max}}(\frac 1n (\frac{\log n}{\log\log n})^2)$.
    This means we can choose $N = O(\frac{1}{\epsilon}\log^2 \frac{1}{\epsilon})$ such that for all $n\ge N$, 
    the minimax regret over the class $\mathcal{P}([0, \theta_{\max}])$ is bounded by $\frac{\epsilon}{6}$. 
    Now, by the previous two theorems, for each $M$, there exists a transformer model $\Gamma$ that can approximate either Robbins or NPMLE (clipped at $M$) up to $\sqrt{\frac{\epsilon}{6}}$ precision uniformly for inputs up to $d$. 
    Given that $\theta\le \theta_{\max}$, we can choose the threshold $M = \theta_{\max}$ in both cases without increasing the regret. 
    Then we have 
    \begin{flalign*}
        \text{Regret}(\Gamma)
        &\le 2(\text{Regret}(\hat{\theta}) + \EE[(\hat{\theta} - \Gamma)^2])
        \nonumber\\
        &\le 2\left(\frac{\epsilon}{6} + \EE[(\hat{\theta}(X) - \Gamma(X))^2\indc{X \le d}]  + \EE[\theta_{\max}^2\indc{X > d}]\right)
        \nonumber\\
        &\le 2\left(\frac{\epsilon}{6} + \frac{\epsilon}{6} + \frac{\epsilon}{6}\right)
        \nonumber\\
        &=\epsilon
    \end{flalign*}

    Finally, to bound the model complexity of a transformer needed for such approximation, 
    we follow the Robbins approximation as per \prettyref{thm:robbins-transformers}. 
    Here, we used an encoder with embedding dimension of $O(\log(\epsilon^{-1}))$, 
    and decoder an multi-layer perceptrons with $O((1+\theta_{\max})^2\epsilon^{-1})$ layers and width 2. 
    This gives a total of $O((1+\theta_{\max})^2\epsilon^{-1} + \log^2(\epsilon^{-1})) = O_{\theta_{\max}}(\epsilon^{-1})$ parameters.  
\end{proof}

\subsection{Identities on worst-case prior}
\begin{proof}[Proof of \prettyref{lmm:worst_prior_mse}]
    We consider the prior $\pi_{\epsilon}\triangleq (1-\epsilon)\pi_! + \epsilon \delta_{\theta_0}$ for some $\theta_0\in A$. 
    Then $\frac{\partial }{\partial \epsilon}\mathsf{mmse}(\pi_{\epsilon})|_{\epsilon = 0} \le 0$ with equality if 
    $\theta_0\in\text{supp}(\pi_{!, A})$. 
    Consider, now, the following form: 
    \begin{flalign*}
    \mathsf{mmse}(\pi) &= \mathbb{E}[\theta^2] - \mathbb{E}_X[\mathbb{E}[\theta | X]^2]
    = \mathbb{E}[\theta^2] - \mathbb{E}_X[\mathbb{E}[\theta | X]^2]
    \\&= \mathbb{E}[\theta^2] - \sum_x \frac{e_{\pi}(x)^2}{m_{\pi}(x)}
    \end{flalign*}
    where $m_{\pi}(x) = \int p(x|\theta)d\pi(\theta)$ and $e_{\pi}(x) = \int \theta p(x|\theta)d\pi(\theta)$ 
    are the probability mass function and posterior mean of $x$, respectively. 

    Now denote $e_{\theta_0}(x) = \theta_0p(x|\theta_0)$ and $m_{\theta_0}(x) = p(x|\theta_0)$.
    Denote also the difference 
    $d(x) \triangleq m_{\theta_0}(x) - m_{\pi_!}(x)$ and 
    $k(x) \triangleq e_{\theta_0}(x) - e_{\pi_!}(x)$. 
    Then 
    \begin{align*}
        & ~\mathsf{mmse}(\pi_{\epsilon})
    \\= & ~\mathbb{E}_{\pi_!}[\theta^2] + \epsilon(\theta_0^2 - \mathbb{E}_{\pi_!}[\theta^2]) 
    - \sum_x \frac{(e_{\pi_!}(x) + \epsilon k(x))^2}{m_{\pi_!}(x) + \epsilon d(x)}
    \end{align*}
    which means the derivative when evaluated at 0: 
    \begin{flalign*}
        0\ge &~\frac{\partial }{\partial \epsilon}\mathsf{mmse}(\pi_{\epsilon})|_{\epsilon = 0}
        \\
        =& ~\theta_0^2 - \mathbb{E}_{\pi_!}[\theta^2] 
        - \sum_x \frac{2m_{\pi_!}(x)e_{\pi_!}(x)k(x) - e_{\pi_!}(x)^2d(x)}{m_{\pi_!}(x)^2}
        \\
        =&~\theta_0^2 - \sum_x \frac{2m_{\pi_!}(x)e_{\pi_!}(x)e_{\theta_0}(x) - e_{\pi_!}(x)^2m_{\theta_0}(x)}{m_{\pi_!}(x)^2} 
        \\
        &\quad -\mathsf{mmse}(\pi_!)
        \\
        =&~\theta_0^2 - 2\sum_x e_{\theta_0}(x) \frac{e_{\pi_!}(x)}{m_{\pi_!}(x)}
        +\sum_x m_{\theta_0}(x)\left(\frac{e_{\pi_!}(x)}{m_{\pi_!}(x)}\right)^2 \\
        &\quad - \mathsf{mmse}(\pi_!)
        \\
        =&~\mathsf{mse}_{\delta_{\theta}}(f_{\pi!}) - \mathsf{mmse}(\pi_!)
    \end{flalign*}
    where the last equality follows from that $\frac{e_{\pi_!}(x)}{m_{\pi_!}(x)} = f_{\pi_!}(x)$. 
    Therefore, the conclusion follows. 
\end{proof}

\section{Pseudocode on Robbins approximation via transformers}\label{app:transf-robbins}
We present a pseudocode in \prettyref{alg:robbins-transformers} on how a transformer can be set up to approximate Robbins, using a formulation that closely mimics the PyTorch module. 
All vectors and matrices use 1-indexing. 
Note that the attention output is $Z = \text{Softmax}(\frac{YW_QW_K^TY^T}{\sqrt{d_k}})YW_V$. 

\begin{algorithm}[tb]
   \caption{Pseudocode that approximates $\hat{\theta}_{\mathsf{Rob}, d, \theta_{\max}}$ using a transformer.}
   \label{alg:robbins-transformers}
\begin{algorithmic}
   \STATE {\bfseries Input:} Inputs $x_1, \cdots, x_n$, $\theta_{\max}$, $d$. 
   \STATE {\bfseries Define: } $d_k = d + 1$, $n_{head} = 1$. 
   \STATE {\bfseries Define: } $D = \max \{100, d_k^2\}$.  
   \STATE {\bfseries Define: } $W_Q = I_{d_k}$, $W_V = \text{diag}(1, 1, \cdots, 1, 0), W_K$. 
   \FOR {$i=1$ {\bfseries to} $d + 1$}
       \FOR {$j=1$ {\bfseries to} $d + 1$}
           \IF{$i=j$}
           \STATE $W_k[i, j] = D$ 
           \ELSIF{$j=i+1$}
           \STATE $W_k[i, j] = D + \sqrt{d+1}\log i$ 
           \ELSE
           \STATE $W_k[i, j] = 0$
           \ENDIF
       \ENDFOR
   \ENDFOR 
   \STATE {\bfseries Define: } AttnLayer = Attn($W_Q, W_K, W_V$). 
   \STATE {\bfseries Define: } $Z=\text{AttnLayer}(Y, Y, Y)$. 
   \STATE $Z' = \mathsf{ReLU}(Y + Z - 1).$
   \STATE $Z_1 = \text{rowsum}(Z')$. 
   \STATE \textbf{return} $\min \{\frac{1}{Z_1} - 1, M\}$. 
\end{algorithmic}
\end{algorithm}

\section{Estimates on resources needed}\label{app:compute}
We briefly detail an estimate the amount of compute we used in producing this work. 

\textbf{Pre-training stage}. 
On 2 Nvidia Volta V100 GPUs and 40 CPUs, our model T24r takes about 28 hours to train, 
while L24r takes about 3 hours (due to its linear complexity). 
Together with other models trained and evaluated during hyperparameters search, we estimate the whole pretraining process to take somewhere between 1000 to 2000 GPU hours. 

\textbf{Synthetic experiments}. 
Each experiment described in \prettyref{sec:synthetic} is done on 2 Nvidia Volta V100 GPUs and 40 CPUs for computation (340 GB total RAM), 
possibly split between different tasks. 
The time taken ranges from a few minutes (Robbins/MLE) to a few days (NPMLE even with reduced number of batches used per prior). 
For non-NPMLE tasks, one of the most computationally intensive task is T24r on $n=2048$ case, 
which takes around 32 GPU hours. 
We note that it is possible to use only CPU for these tasks, although it will take longer to run. 

\textbf{Real experiments}.
Evaluation of each estimator (including NPMLE) on the sports dataset in \prettyref{sec:real} takes no more than a few minutes, where 2 Intel CPUs (8 GB) would suffice. 
Evaluation of BookCorpusOpen dataset can take a few hours (T24r) to one day (NPMLE). 

\section{Further analysis on synthetic experiments}\label{app:synthetic}
\subsection{Performance}\label{app:syntetic-performance}
We recall that our synthetic experiments are measuring regret w.r.t. sequence length for both the neural and worst-prior. 
In the main section, we show a plot of how the regret decreases with sequence length; 
here, we provide a more comprehensive result on the Plackett-Luce rankings in 
\prettyref{tab:pl-synthetic}, 
along with the $p$-value by pairwise $t$-test of T24r against relevant classical baselines, 
as per \prettyref{tab:tstat-sythetetic-24}. 
From the $p$-value we conclude that the transformers outperform other baselines by a significant margin on various experiments.  
(except in a handful of cases). 

\begin{table*}[ht]
\caption{Plackett-Luce coefficients of estimators' regrets on synthetic experiments. The coefficient of MLE is set to 0 throughout.}
\label{tab:pl-synthetic}
\vskip 0.15in
\begin{center}
\begin{small}
\begin{sc}
\begin{tabular}{lcccccc}
\toprule
Experiments & GS & Robbins & ERM & NPMLE & T24r & L24r \\
\midrule 
Neural-128 & -0.004 & -3.384 & 0.979 & 4.542 & \textbf{7.104} & 7.005\\
Neural-256 & -0.024 & -3.271 & 1.713 & 5.194 & \textbf{7.402} & 7.192\\
Neural-512 & -0.045 & -3.157 & 2.484 & 5.894 & \textbf{7.533} & 7.157\\
Neural-1024 & -0.067 & -3.015 & 3.120 & 6.703 & \textbf{7.525} & 7.082\\
Neural-2048 & -0.092 & -2.834 & 3.461 & \textbf{7.515} & 7.487 & 7.009\\
\midrule
WP-128 & - & -4.900 & -2.429 & 2.475 & 4.942 & \textbf{7.407}\\
WP-256 & - & -2.466 & 2.469 & 4.934 & 7.403 & \textbf{9.874}\\
WP-512 & - & -2.724 & 2.725 & 5.452 & \textbf{8.581} & 8.556\\
WP-1024 & - & -2.472 & 2.468 & 4.937 & \textbf{9.881} & 7.409\\
WP-2048 & - & -2.462 & 2.462 & 4.925 & \textbf{9.858} & 7.393\\
\midrule
Multn-128 & 0.462 & -4.808 & -2.500 & 3.186 & 5.495 & \textbf{7.796}\\
Multn-256 & 0.443 & -4.733 & -2.413 & 3.192 & 5.510 & \textbf{7.827}\\
Multn-512 & 0.502 & -2.646 & 2.230 & 4.847 & 7.301 & \textbf{9.753}\\
Multn-1024 & 0.455 & -2.527 & 3.212 & 5.550 & 7.896 & \textbf{9.955}\\
Multn-2048 & 0.471 & -2.720 & 3.419 & 5.939 & 8.324 & \textbf{9.320}\\
\bottomrule
\end{tabular}
\end{sc}
\end{small}
\end{center}
\vskip -0.1in
\end{table*}

\begin{table*}[ht]
\caption{$\mathbb{P}[\mathsf{Regret}\text{(T24r)} > \mathsf{Regret}\text{(Classical)}]$ obtained via paired $t$-test.}
\label{tab:tstat-sythetetic-24}
\vskip 0.15in
\begin{center}
\begin{small}
\begin{sc}
\begin{tabular}{lccccc}
\toprule
Experiments & MLE & GS & Robbins & ERM & NPMLE \\
\midrule 
Neural-128 & $<$ 1e-100 & $<$1e-100 & $<$ 1e-100 & $<$ 1e-100 & $<$ 1e-100\\
Neural-256 & $<$ 1e-100 & $<$1e-100 & $<$ 1e-100 & $<$ 1e-100 & $<$ 1e-100\\
Neural-512* & $<$ 1e-100 & $<$1e-100 & $<$ 1e-100 & $<$ 1e-100 & $<$ 1e-100\\
Neural-1024 & $<$ 1e-100 & $<$1e-100 & $<$ 1e-100 & $<$ 1e-100 & 0.0252\\
Neural-2048 & $<$ 1e-100 & $<$1e-100 & $<$ 1e-100 & $<$ 1e-100 & $>$ 1 - 1e-100\\
\midrule
WP-128 & $<$ 1e-100 & - & $<$ 1e-100 & $<$ 1e-100 & $<$ 1e-100\\
WP-256 & $<$ 1e-100& - & $<$ 1e-100 & $<$ 1e-100 & $<$ 1e-100\\
WP-512 & $<$ 1e-100 & - & $<$ 1e-100 & $<$ 1e-100 & $<$ 1e-100\\
WP-1024 & $<$ 1e-100 & - & $<$ 1e-100 & $<$ 1e-100 & $<$1e-100\\
WP-2048 & $<$ 1e-100 & - & $<$ 1e-100 & $<$ 1e-100 & $<$1e-100\\
\midrule
Multn-128 & $<$1e-100 & $<$ 1e-100 & $<$ 1e-100 & $<$ 1e-100 & $<$ 1e-100 \\
Multn-256 & $<$1e-100 & $<$ 1e-100 & $<$ 1e-100 & $<$ 1e-100 & $<$ 1e-100 \\
Multn-512 & $<$1e-100 & $<$ 1e-100 & $<$ 1e-100 & $<$ 1e-100 & $<$ 1e-100 \\
Multn-1024 & $<$1e-100 & $<$ 1e-100 & $<$ 1e-100 & $<$ 1e-100 & $<$ 1e-100 \\
Multn-2048 & $<$1e-100 & $<$ 1e-100 & $<$ 1e-100 & $<$ 1e-100 & $<$ 1e-100 \\
\bottomrule
\end{tabular}
\end{sc}
\end{small}
\end{center}
\vskip -0.1in
\end{table*}

\begin{table*}[ht]
\caption{$\mathbb{P}[\mathsf{Regret}\text{(L24r)} > \mathsf{Regret}\text{(Classical)}]$ obtained via paired $t$-test.}
\label{tab:tstat-sythetetic-l24r}
\vskip 0.15in
\begin{center}
\begin{small}
\begin{sc}
\begin{tabular}{lccccc}
\toprule
Experiments & MLE & GS & Robbins & ERM & NPMLE \\
\midrule 
Neural-128 & $<$ 1e-100 & $<$1e-100 & $<$ 1e-100 & $<$ 1e-100 & $<$ 1e-100\\
Neural-256 & $<$ 1e-100 & $<$1e-100 & $<$ 1e-100 & $<$ 1e-100 & $<$ 1e-100\\
Neural-512* & $<$ 1e-100 & $<$1e-100 & $<$ 1e-100 & $<$ 1e-100 & $<$ 1e-100\\
Neural-1024 & $<$ 1e-100 & $<$1e-100 & $<$ 1e-100 & $<$ 1e-100 & 1 - 1.92e-6\\
Neural-2048 & $<$ 1e-100 & $<$1e-100 & $<$ 1e-100 & $<$ 1e-100 & $>$ 1 - 1e-100\\
\midrule
WP-128 & $<$ 1e-100 & - & $<$ 1e-100 & $<$ 1e-100 & $<$ 1e-100\\
WP-256 & $<$ 1e-100& - & $<$ 1e-100 & $<$ 1e-100 & $<$ 1e-100\\
WP-512 & $<$ 1e-100 & - & $<$ 1e-100 & $<$ 1e-100 & $<$ 1e-100\\
WP-1024 & $<$ 1e-100 & - & $<$ 1e-100 & $<$ 1e-100 & $<$1e-100\\
WP-2048 & $<$ 1e-100 & - & $<$ 1e-100 & $<$ 1e-100 & $<$1e-100\\
\midrule
Multn-128 & $<$1e-100 & $<$ 1e-100 & $<$ 1e-100 & $<$ 1e-100 & $<$ 1e-100 \\
Multn-256 & $<$1e-100 & $<$ 1e-100 & $<$ 1e-100 & $<$ 1e-100 & $<$ 1e-100 \\
Multn-512 & $<$1e-100 & $<$ 1e-100 & $<$ 1e-100 & $<$ 1e-100 & $<$ 1e-100 \\
Multn-1024 & $<$1e-100 & $<$ 1e-100 & $<$ 1e-100 & $<$ 1e-100 & $<$ 1e-100 \\
Multn-2048 & $<$1e-100 & $<$ 1e-100 & $<$ 1e-100 & $<$ 1e-100 & $<$ 1e-100 \\
\bottomrule
\end{tabular}
\end{sc}
\end{small}
\end{center}
\vskip -0.1in
\end{table*}

\newpage

\section{Further Analysis on real data experiments}\label{app:real}
\subsection{Comparison with deep-learning Bayesian models}\label{app:bayesian}
We describe why Bayesian models do not solve the real data experiments based on our EB modelling. 
Fundamentally, this can be seen from the following two aspects: unsupervised vs supervised, and adaptability to prior misspecification. 

\textbf{Supervised vs unsupervised}. 
In EB our predictions are unsupervised at test time: we estimate $\theta_1, \cdots, \theta_n$ based only on inputs 
$X_1, \cdots, X_n$, 
with no prompts in the form $(X, \theta)$, and no assumption placed on $\pi$ from which $\theta$ is drawn. 
This is especially evident in our experimental setting in \prettyref{sec:real} where we make predictions $Y_i$'s based only on $X_i$'s. 

On the other hand, in simulation based inference (SBI) one samples many $(\theta, X)$ pairs from the prior $\pi$. 
For example, \cite{zammit2024neural} fixes a prior on $\theta$, generate many pairs $(\theta, X)$, 
and train a neural network to compute posterior $p_{\theta | X}$ by minimizing a suitable entropy loss. 
Note that in EB, such information on $\theta$ is unknown at test time, 
and the prior is itself unknown and has to be estimated from the input sequence $X_1, \cdots, X_n$ alone.

\textbf{Adaptability to prior misspecification}. 
In EB, there is no assumption at test time on the prior from which the labels $\theta$ are drawn. 
In the case of transformers, for example, we evaluate them on priors unseen during the pretraining phase, (i.e. multinomial prior-on-priors and worst-case prior). 
Our pre-trained transformers are willing to estimate the prior at test time and adapt to it to predict $\mathbb{E}[\theta | X]$.
Not only do they outperform the classical estimators: EB (NPMLE, Robbins, ERM-monotone) and non-EB (MLE, gold-standard), 
they also leverage longer sequences to produce better estimates (i.e. length generalization). 
Note that these classical EB algorithms also provably achieve decreasing regret upon longer sequences (\cite[Theorem 2]{polyanskiy2021sharp}, \cite[Theorem 3]{jana2022optimal}, \cite[Theorem 1]{jana2023empirical} regardless of the prior (provided it fulfills certain light-tailedness condition), which is an advantage EB estimators enjoy over the non-EB estimators. 

In contrast, the problem of different prior at training and test time (i.e. model misspecification) presents significant challenges in Bayesian inference, as acknowledged in \cite[Appendix S2.8]{zammit2024neural} and \cite[page 9 before section 3D]{cranmer2020frontier}. 
The discussion of the gold-standard estimators in \prettyref{app:worstprior} illustrates the fundamental limit of non-EB estimators in this respect. 

\textbf{Examples. }
Some recent works of Bayesian inference include TabPFN \cite{hollmann2025accurate}, ACE \cite{chang2024amortized}, Simformer \cite{gloeckler2024all} and Generative Bayesian Computation (GBC) \cite{polson2024generative}. 
These are all supervised in the following sense: 
\begin{itemize}
    \item TabPFN and ACE receive prompts containing `input, label' pairs at test time, 
    as illustrated in \cite[Figure 1(a)]{hollmann2025accurate} and \cite[(1)]{chang2024amortized}; 

    \item Simformer requires simulation for each task (e.g. Appendices A2.2 and A3.2 of \cite{gloeckler2024all} detail the training protocols that include the number of simulations for each task)

    \item GBC uses simulation (e.g. Section 2 of \cite{polson2024generative} mentions the need of generative AI to simulate from all distributions). 
\end{itemize}

Note that in the example of the real dataset, $\theta$ could even be impossible to obtain 
(for example, in the NHL or MLB dataset, $\theta$ represents the innate ability of an athlete and $X$ the number of goals scored). 

We acknowledge that these `input-label' pairs could have been obtained in an experimental setting different than ours. E.g. in the NHL or MLB dataset, prompts in the form of $(X_{:, j - t}, X_{:, j - t + 1})$ representing a player's score in seasons $j - t$ and $j - t + 1$ can be obtained, where $t = 1, 2, \cdots, T$ are lookback windows (thus here the label is not $\theta$ but an independent set of observations). 
These and $X_{:, j}$ can be given to these deep-learning Bayesian models before asking them to predict $X_{:, j + 1}$. 
Note that in our EB setting, $T = 0$, which these frameworks are incapable of. 
Constructing EB estimators to leverage extra information when $T\ge 1$ is beyond the scope of this paper. 

\subsection{More comparisons on the RMSE metric}\label{app:rmse}
In the main section, we on reported the percentage improvement in RMSE compared to MLE. 
We supplement our finding on RMSE by showing a Plackett Luce ELO coefficient of RMSEs of the estimators in \prettyref{tab:elo_rmse}, 
which shows that the transformers are consistently ranked at the top. 
We will also show the results of paired $t$-test on \prettyref{tab:tstat_rmse} in terms of $\mathbb{P}[\text{RMSE(T24r)} > \text{RMSE(baselines)}]$, 
and \prettyref{tab:tstat_rmse_l24r} in terms of $\mathbb{P}[\text{RMSE(L24r)} > \text{RMSE(baselines)}]$, 
which shows that T24r and L24r consistently gives significant improvement over the baselines (except for T24r on BookCorpusOpen). 

\begin{table*}[ht]
\caption{Plackett-Luce coefficients of estimators' RMSE on real datasets. The coefficient of MLE is set to 0 throughout.}
\label{tab:elo_rmse}
\vskip 0.15in
\begin{center}
\begin{small}
\begin{sc}
\begin{tabular}{lccccc}
\toprule
Dataset & Robbins & ERM & NPMLE & T24r & L24r \\
\midrule
NHL (all) & -2.511 & 1.453 & 3.348 & 3.844 & \textbf{4.805}\\
NHL (defender) & -1.730 & 1.617 & 3.511 & 4.222 & \textbf{4.720}\\
NHL (center) & -2.710 & 0.439 & 2.120 & 3.020 & \textbf{3.197}\\
NHL (winger) & -3.308 & 0.654 & 2.236 & 3.221 & \textbf{3.814}\\
\midrule 
MLB (batting) & -3.675 & 0.862 & 3.994 & \textbf{6.995} & 5.257\\
MLB (pitching) & -3.466 & 0.543 & 4.273 & \textbf{7.520} & 6.166 \\
\midrule 
BookCorpusOpen  & -0.018 & 1.523 & \textbf{2.312} & 1.610 & 1.558 \\
\bottomrule
\end{tabular}
\end{sc}
\end{small}
\end{center}
\vskip -0.1in
\end{table*}

\begin{table*}[ht]
\caption{$\mathbb{P}[\text{RMSE(T24r)} > \text{RMSE(baselines)}]$ obtained via paired $t$-test.}
\label{tab:tstat_rmse}
\vskip 0.15in
\begin{center}
\begin{small}
\begin{sc}
\begin{tabular}{lcccc}
\toprule
Dataset & MLE & Robbins & ERM & NPMLE\\
\midrule
NHL (all) & 1.51e-07 & 3.43e-11 & 1.96e-08 & 4.54e-04\\
NHL (defender) & 1.05e-08 & 1.61e-10 & 3.46e-08 & 0.00225\\
NHL (center) & 8.23e-08 & 1.79e-11 & 4.87e-08 & 2.95e-06\\
NHL (winger) & 9.97e-08 & 7.36e-13 & 1.71e-07 & 3.92e-04\\
\midrule
MLB-Bat & 6.57e-17 & 9.89e-19 & 2.42e-11 & 1.20e-06\\
MLB-Pitch & 1.38e-15 & 3.30e-20 & 3.66e-13 & 7.26e-09\\
\midrule
BookCorpusOpen & $<$ 1e-100 & $<$ 1e-100 & 1 - 4.53e-4 & $>$ 1- 1e-100 \\
\bottomrule
\end{tabular}
\end{sc}
\end{small}
\end{center}
\vskip -0.1in
\end{table*}

\begin{table*}[ht]
\caption{$\mathbb{P}[\text{RMSE(L24r)} > \text{RMSE(baselines)}]$ obtained via paired $t$-test.}
\label{tab:tstat_rmse_l24r}
\vskip 0.15in
\begin{center}
\begin{small}
\begin{sc}
\begin{tabular}{lcccc}
\toprule
Dataset & MLE & Robbins & ERM & NPMLE\\
\midrule
NHL (all) & 1.01e-06 & 2.07e-11 & 9.38e-07 & 2.35e-04\\
NHL (defender) & 6.47e-08 & 5.95e-11 & 4.84e-07 & 0.00203\\
NHL (center) & 1.03e-06 & 9.78e-12 & 4.89e-07 & 6.05e-04\\
NHL (winger) & 1.57e-07 & 6.74e-13 & 5.09e-07 & 7.76e-04\\
\midrule
MLB batting & 2.78e-16 & 9.79e-19 & 7.28e-11 & 1.61e-04\\
MLB pitching & 1.66e-15 & 3.31e-20 & 1.15e-12 & 1.75e-06\\
\midrule
BookCorpusOpen & $<$ 1e-100 & $<$ 1e-100 & 8.06e-95 & 1.76e-08\\
\bottomrule
\end{tabular}
\end{sc}
\end{small}
\end{center}
\vskip -0.1in
\end{table*}

\subsection{Evaluation on the MAE metric}\label{app:mae}
We compare the Mean Absolute Error (MAE) of various estimators to supplement the RMSE results.  
Apart from the average percentage improvement of each algorithm over the MLE 
as per \prettyref{tab:dataset_mae}, 
we also display the $p$-values based on paired $t$-test of transformers vs other algorithms in \prettyref{tab:tstat_mae}, 
and \prettyref{tab:elo_mae}. 
In \prettyref{fig:hockey_violin_mae}, \prettyref{fig:batting_violin_mae}, 
\prettyref{fig:pitching_violin_mae} and \prettyref{fig:bookcorpus_violin_mae}, 
we also include the violin plots.

\begin{table*}[ht]
\caption{95\% confidence interval of the percentage improvement of MAE by each algorithm over MLE.}
\label{tab:dataset_mae}
\vskip 0.15in
\begin{center}
\begin{small}
\begin{sc}
\begin{tabular}{lccccc}
\toprule
Dataset & Robbins & ERM & NPMLE & T24r & L24r \\
\midrule
NHL (all) & -20.14 $\pm$ 4.44 & -0.20 $\pm$ 0.60 & 0.88 $\pm$ 0.57 & 0.93 $\pm$ 0.53 & \textbf{1.12 $\pm$ 0.67} \\
NHL (defender) & -13.38 $\pm$ 4.24 & 1.44 $\pm$ 1.16 & 3.26 $\pm$ 1.07  & 3.09 $\pm$ 1.10 & \textbf{3.64 $\pm$ 1.33} \\
NHL (center) & -41.43 $\pm$ 9.21 & -0.65 $\pm$ 0.86 & 2.26 $\pm$ 0.82 & 2.80 $\pm$ 0.80 & \textbf{2.94 $\pm$ 0.96} \\
NHL (winger) & -32.43 $\pm$ 7.23 & -0.12 $\pm$ 0.78 & 1.43 $\pm$ 0.63 & 1.82 $\pm$ 0.62  & \textbf{2.14 $\pm$ 0.67} \\
\midrule
Baseball (batting) & -49.64 $ \pm$ 4.74 & 0.49 $ \pm$ 0.24 & 1.57 $ \pm$ 0.21 & 1.72 $\pm$ 0.20 & \textbf{1.85 $ \pm$ 0.20}\\
Baseball (pitching) & -38.16 $\pm$ 2.87 & 0.08 $\pm$ 0.21 & 1.55 $\pm$ 0.19 & \textbf{1.76 $\pm$ 0.18} & 1.70 $\pm$ 0.17 \\
\midrule
BookCorpusOpen  & 28.05 $\pm$ 0.14 & 29.54 $\pm$ 0.10 & 29.65 $\pm$ 0.10 & 23.17 $\pm$ 0.21 & \textbf{29.72 $\pm$ 0.12} \\
\bottomrule
\end{tabular}
\end{sc}
\end{small}
\end{center}
\vskip -0.1in
\end{table*}

\begin{table*}[ht]
\caption{Plackett-Luce coefficients of estimators' MAE on real datasets. The coefficient of MLE is set to 0 throughout.}
\label{tab:elo_mae}
\vskip 0.15in
\begin{center}
\begin{small}
\begin{sc}
\begin{tabular}{lccccc}
\toprule
Dataset & Robbins & ERM & NPMLE & T24r & L24r \\
\midrule
NHL (all) & -2.849 & 0.078 & 1.844 & 1.680 & \textbf{2.268}\\
NHL (defender) & -2.181 & 0.810 & 2.423 & 2.104 & \textbf{2.679}\\
NHL (center) & -3.139 & -0.401 & 1.556 & 2.546 & \textbf{2.654}\\
NHL (winger) & -2.817 & 0.198 & 1.356 & 2.334 & \textbf{2.842}\\

\midrule
MLB batting & -3.909 & 0.872 & 4.309 & 4.645 & 6.386 \\
MLB pitching & -3.977 & 0.002 & 3.461 & 6.627 & 5.166 \\
\midrule
BookCorpusOpen  & 3.585 & 4.191 & \textbf{4.395} & 2.964 & 4.131 \\
\bottomrule
\end{tabular}
\end{sc}
\end{small}
\end{center}
\vskip -0.1in
\end{table*}

\begin{table*}[ht]
\caption{$\mathbb{P}[\text{MAE(T24r)} > \text{MAE(baselines)}]$ obtained via paired $t$-test.}
\label{tab:tstat_mae}
\vskip 0.15in
\begin{center}
\begin{small}
\begin{sc}
\begin{tabular}{lcccc}
\toprule
Dataset & MLE & Robbins & ERM & NPMLE\\
\midrule
NHL (all) & 0.00124 & 6.31e-10 & 2.83e-06 & 0.195\\
NHL (defender) & 2.35e-06 & 1.36e-09 & 0.000253 & 0.929\\
NHL (center) & 1.35e-07 & 3.1e-10 & 1.12e-08 & 0.000236\\
NHL (winger) & 1.33e-06 & 6.63e-10 & 1.88e-06 & 4.58e-06\\

\midrule
MLB-Bat & 1.63e-16 & 2.39e-19 & 4.60e-10 & 0.00662\\
MLB-Pitch & 4.73e-18 & 4.38e-20 & 1.25e-14 & 1.85e-06\\
\midrule
BookCorpusOpen & $<$ 1e-100 & $>$ 1- 1e-100 & $>$ 1- 1e-100 & $>$ 1- 1e-100\\
\bottomrule
\end{tabular}
\end{sc}
\end{small}
\end{center}
\vskip -0.1in
\end{table*}

\begin{table*}[ht]
\caption{$\mathbb{P}[\text{MAE(L24r)} > \text{MAE(Classical)}]$ obtained via paired $t$-test.}
\label{tab:tstat_mae_l24r}
\vskip 0.15in
\begin{center}
\begin{small}
\begin{sc}
\begin{tabular}{lcccc}
\toprule
Dataset & MLE & Robbins & ERM & NPMLE\\
\midrule
NHL (all) & 0.00189 & 4.42e-10 & 2.04e-06 & 0.0106\\
NHL (defender) & 3.88e-06 & 4.13e-10 & 9.28e-05 & 0.0337\\
NHL (center) & 1.56e-06 & 2.89e-10 & 7.49e-08 & 0.00198\\
NHL (winger) & 3.55e-07 & 5.37e-10 & 3.88e-07 & 1.04e-06\\
\midrule
MLB batting & 3.92e-17 & 2.10e-19 & 8.51e-11 & 5.43e-05\\
MLB pitching & 1.84e-18 & 4.61e-20 & 1.52e-14 & 6.20e-04\\
\midrule
BookCorpusOpen & $<$1e-100 & 1.43e-55 & 2.63e-26 & 2.18e-16\\
\bottomrule
\end{tabular}
\end{sc}
\end{small}
\end{center}
\vskip -0.1in
\end{table*}

\begin{figure}
\begin{subfigure}[b]{0.24\linewidth}
\begin{center}
\centerline{\includegraphics[width=\linewidth]{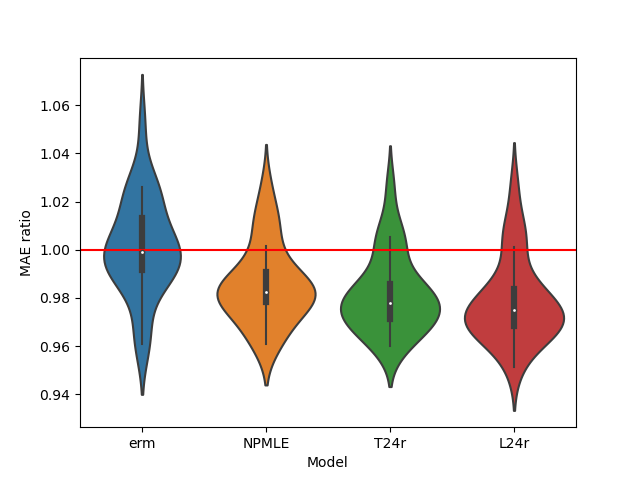}}
    \caption{NHL Hockey}
    \label{fig:hockey_violin_mae}
    \end{center}
\vskip -0.2in
\end{subfigure}
\begin{subfigure}[b]{0.24\linewidth}
\begin{center}
\centerline{\includegraphics[width=\linewidth]{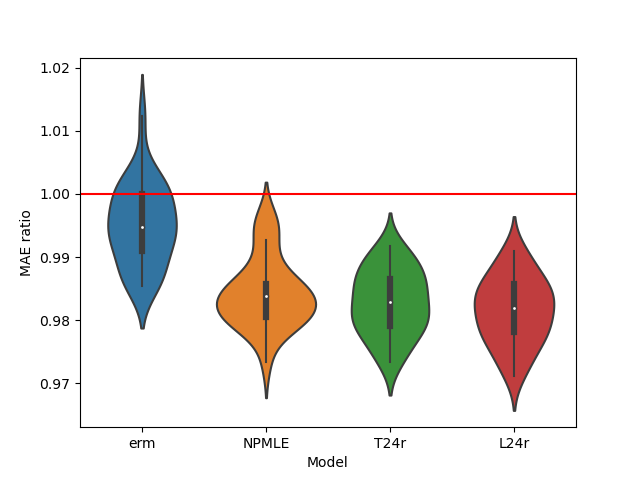}}
    \caption{MLB Batting}
    \label{fig:batting_violin_mae}
    \end{center}
\vskip -0.2in
\end{subfigure}
\begin{subfigure}[b]{0.24\linewidth}
\begin{center}
\centerline{\includegraphics[width=\linewidth]{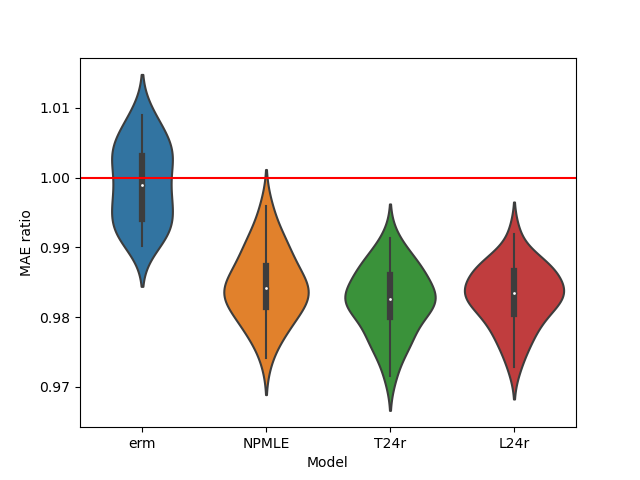}}
    \caption{MLB Pitching}
    \label{fig:pitching_violin_mae}
    \end{center}
\end{subfigure}
\begin{subfigure}[b]{0.24\linewidth}
\begin{center}
\centerline{\includegraphics[width=\linewidth]{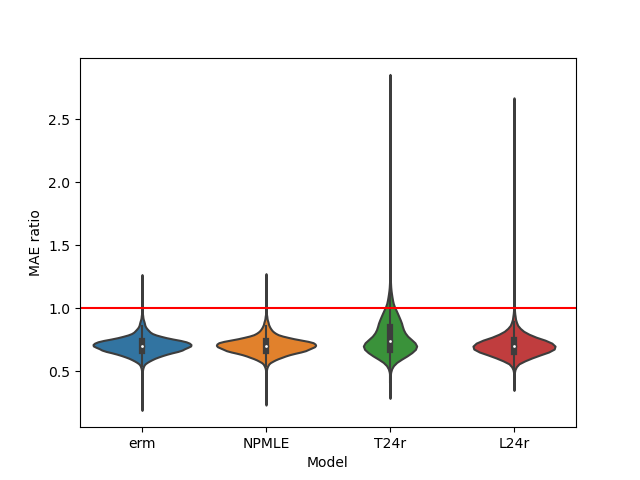}}
    \caption{BookCorpusOpen.}
    \label{fig:bookcorpus_violin_mae}
    \end{center}
\end{subfigure}
\caption{Violin plots of MAE ratio of ERM-monotone (blue), NPMLE (orange), T24r (green) and L24r (red) over MLE over
various datasets. 
Note that T24r and L24r show a general improvement for the MLB datasets, but for bookcorpus they are generally affected by the outliers. 
The horizontal line at 1.0 indicates the threshold of MLE’s RMSE. Robbins is not
shown due to the wide variance.}
\end{figure}

\end{document}